\newif\ifarxiv
\arxivtrue

\documentclass[letterpaper, 10 pt, conference]{ieeeconf}  

\IEEEoverridecommandlockouts                              

\usepackage{amsmath} 
\usepackage{amssymb}  
\usepackage{graphicx}
\usepackage{algorithm}
\usepackage{algpseudocode}
\usepackage[utf8]{inputenc}
\usepackage{booktabs}
\usepackage{multirow}
\usepackage{siunitx}
\let\labelindent\relax
\usepackage{enumitem}
\usepackage[hidelinks]{hyperref}
\usepackage{xurl}

\newtheorem{theorem}{Theorem}
\newtheorem{definition}{Definition}
\newtheorem{proposition}{Proposition}
\newtheorem{lemma}{Lemma}

\usepackage[backend=biber, style=ieee, doi=false, url=false, isbn=false, sorting=none]{biblatex}

\AtEveryBibitem{
 \clearlist{language}
}
\title{\LARGE \bf
DRO-EDL-MPC: Evidential Deep Learning-Based Distributionally Robust Model Predictive Control for Safe Autonomous Driving
}

\author{Hyeongchan Ham$^{1}$ and Heejin Ahn$^{1*}$
\thanks{This research was supported by Basic Science Research Program through the National Research Foundation of Korea (NRF) funded by the Ministry of Education(RS-2025-25407010), IITP(Institute of Information \& Coummunications Technology Planning \& Evaluation)-ITRC(Information Technology Research Center) grant funded by the Korea government(Ministry of Science and ICT)(IITP-2026-RS-2023-00259991), and Hyundai Motor Chung Mong-Koo Foundation.}
\thanks{$^{1}$All authors are with the School of Electrical Engineering, Korea Advanced Institute of Science and Technology (KAIST), Yuseong-gu, 34141 Daejeon, Republic of Korea (e-mail: hyeongchan.ham@kaist.ac.kr;heejin.ahn@kaist.ac.kr).}
\thanks{Project Page: \url{https://dro-edl-mpc.github.io}}
}

\begin{document}

\maketitle
\thispagestyle{empty}
\pagestyle{empty}

\begin{abstract}

Safety is a critical concern in motion planning for autonomous vehicles. 
Modern autonomous vehicles rely on neural network-based perception, but making control decisions based on these inference results poses significant safety risks due to inherent uncertainties. To address this challenge, we present a distributionally robust optimization (DRO) framework that accounts for both aleatoric and epistemic perception uncertainties using evidential deep learning (EDL). Our approach introduces a novel ambiguity set formulation based on evidential distributions that dynamically adjusts the conservativeness according to perception confidence levels. We integrate this uncertainty-aware constraint into model predictive control (MPC), proposing the DRO-EDL-MPC algorithm with computational tractability for autonomous driving applications. Validation in the CARLA simulator demonstrates that our approach maintains efficiency under high perception confidence while enforcing conservative constraints under low confidence.

\end{abstract}

\section{INTRODUCTION}

Autonomous driving systems have attracted significant attention due to their potential impacts \cite{yurtsever_survey_2020}. These systems rely on three fundamental technologies: perception to interpret the environment, planning to determine optimal trajectories, and control to execute driving actions. One of the critical challenges in developing safe autonomous driving systems lies in the inherent uncertainties of the perception module, which propagate through the decision-making pipeline \cite{sun_toward_2024}. These uncertainties stem from noises in the data, known as \textit{data uncertainty} or \textit{aleatoric uncertainty}, and the lack of knowledge during the training procedure, known as \textit{model uncertainty} or \textit{epistemic uncertainty} \cite{kendall_what_2017}.

To address these uncertainties, researchers have proposed various approaches for safe and robust decision-making in autonomous driving systems. For handling data uncertainty specifically, risk metrics-based approaches have been developed, including chance-constrained optimization \cite{nemirovski_convex_2007, ren_chance-constrained_2023} and conditional value-at-risk (CVaR) methods \cite{rockafellar_optimization_2000, majumdar_how_2017}. While effective in principle, these approaches assume that perception models accurately represent data uncertainty distributions. In practice, perception models suffer from inherent model uncertainty arising from limited training data, architectural constraints, and optimization challenges, leading to incomplete knowledge representations \cite{kendall_what_2017}. 

\begin{figure}[t]
\centering
\includegraphics[width=\linewidth]{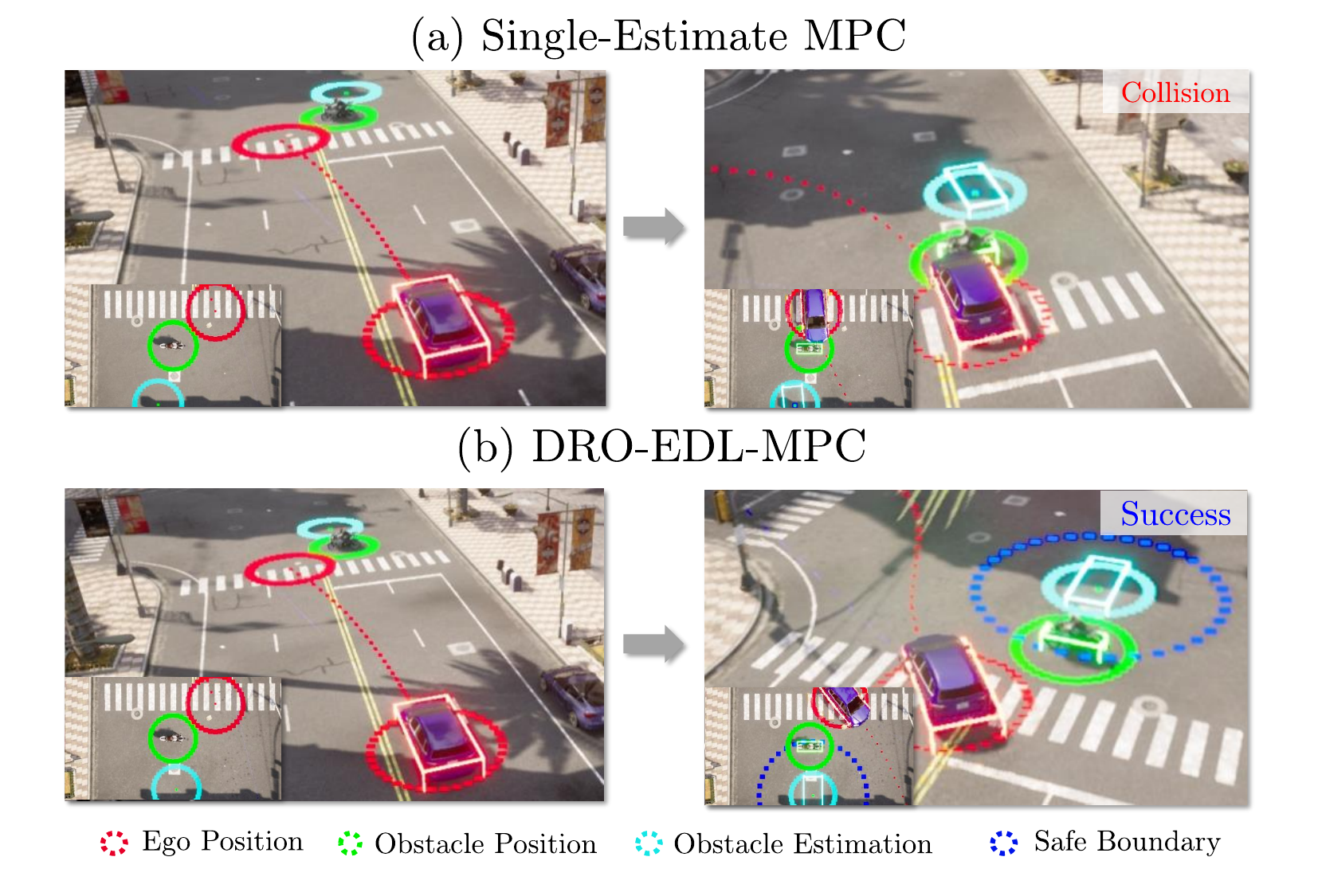}
\vspace*{-0.3in}
\caption{Example illustration of the Single-Estimate MPC and our proposed DRO-EDL-MPC methods in an uncertain perception scenario. (a) The Single-Estimate prediction (light blue circle) deviates significantly from the actual obstacle location (green circle), leading to collision. (b) In contrast, our DRO-EDL-MPC employs uncertainty-aware distributionally robust safety constraint (dark blue circle) and successfully avoids collision.}
\label{fig:MAIN}
\vspace*{-0.3in}
\end{figure}

Distributionally robust optimization (DRO) has emerged as a promising framework to account for such model uncertainty by optimizing against worst-case distributions within a predefined set of candidate distributions, called an ambiguity set \cite{wiesemann_distributionally_2014}. Ambiguity set formulations can be categorized into two main approaches: discrepancy-based \cite{ben-tal_robust_2013, dellaporta_distributionally_2024, hakobyan_distributionally_2023, safaoui_distributionally_2024, aolaritei_wasserstein_2023, long_sensor-based_2025} and moment-based \cite{delage_distributionally_2010, wang_signal-devices_2025}. Discrepancy-based approaches define an ambiguity set as a set of distributions within fixed radii in terms of discrepancy metrics, such as $\phi$-divergence \cite{ben-tal_robust_2013, dellaporta_distributionally_2024} or Wasserstein distance \cite{hakobyan_distributionally_2023, safaoui_distributionally_2024, aolaritei_wasserstein_2023, long_sensor-based_2025}.
Although \cite{esfahani_data-driven_2017} introduces a sample-dependent rule for adjusting this radius, such methods are widely reported to remain overly conservative in practice \cite{esfahani_data-driven_2017, safaoui_distributionally_2024, hakobyan_wasserstein_2022}. 
Also, moment-based approaches construct ambiguity sets by placing confidence bounds on sample moments and shrink these sets as the number of samples increases \cite{delage_distributionally_2010}.
However, such sample-dependent adaptability is undesirable in robotics, where real-time constraints impose an upper bound on sensor sampling frequency. In practice, the sample size is typically fixed at this limit, eliminating adaptability of the ambiguity set.

Recent work has shifted toward machine learning–based inference to achieve adaptive conservativeness without relying on empirical sample distributions.
For instance, \cite{wang_online-learning-based_2024} adapts the ambiguity set radius using a Dirichlet Process Mixture Model (DPMM); however, as a nonparametric model, DPMM requires a large number of observations to obtain accurate estimates and incurs high computational cost.
Another line of work \cite{wang_signal-devices_2025} employs evidential deep learning (EDL) \cite{amini_deep_2020}, which learns an evidential posterior (referred to as the \textit{evidential distribution}) rather than directly outputting point predictions.
This evidential posterior represents the parameters of the likelihood distribution of the target, and the method in \cite{wang_signal-devices_2025} leverages the variance of the likelihood variance. However, this higher-order variance becomes undefined for highly uncertain predictions and therefore restricts applicability to confident cases.

We propose a principled framework for constructing ambiguity sets that leverages this higher-order evidential distribution, which addresses the limitation of previous EDL-based approaches by directly integrating the evidential distribution. Our approach defines ambiguity sets by establishing boundaries where the cumulative probability of the evidential distribution reaches a predefined threshold. This construction yields ambiguity sets with clear probabilistic semantics: each set contains the true uncertainty distribution of the data with a predetermined confidence level. Unlike discrepancy-based methods that employ static distance metrics irrespective of model confidence, or moment-based approaches that only constrain particular statistical moments, our method dynamically adjusts the ambiguity set size according to the model uncertainty. 
Furthermore, in contrast to previous work \cite{wang_signal-devices_2025}, our approach is broadly applicable to uncertain scenarios.

We implement this EDL-based ambiguity set formulation within the distributionally robust safety constraint, called DR-EDL-CVaR, and this constraint is employed to our proposed DRO-EDL-MPC algorithm for autonomous driving. As shown in Fig. \ref{fig:MAIN}, this algorithm adapts conservativeness based on perception confidence while maintaining computational tractability.
We validate the algorithm in the CARLA simulator \cite{dosovitskiy_carla_2017}.

Our main contributions are as follows:
\begin{itemize}
\item We propose DR-EDL-CVaR, a distributionally robust safety constraint leveraging the cumulative probability of the evidential distribution to construct an uncertainty-aware ambiguity set. 
This provides an informative representation of the uncertainties.
\item We introduce DRO-EDL-MPC, a computationally tractable motion planning algorithm that incorporates DR-EDL-CVaR.
\item We validate DRO-EDL-MPC in the CARLA simulator, demonstrating its effectiveness in enhancing safety under various perception uncertainty conditions.
\end{itemize}

\section{Problem Formulation}
\begin{figure}[t]
\centering
\includegraphics[width=0.8\linewidth]{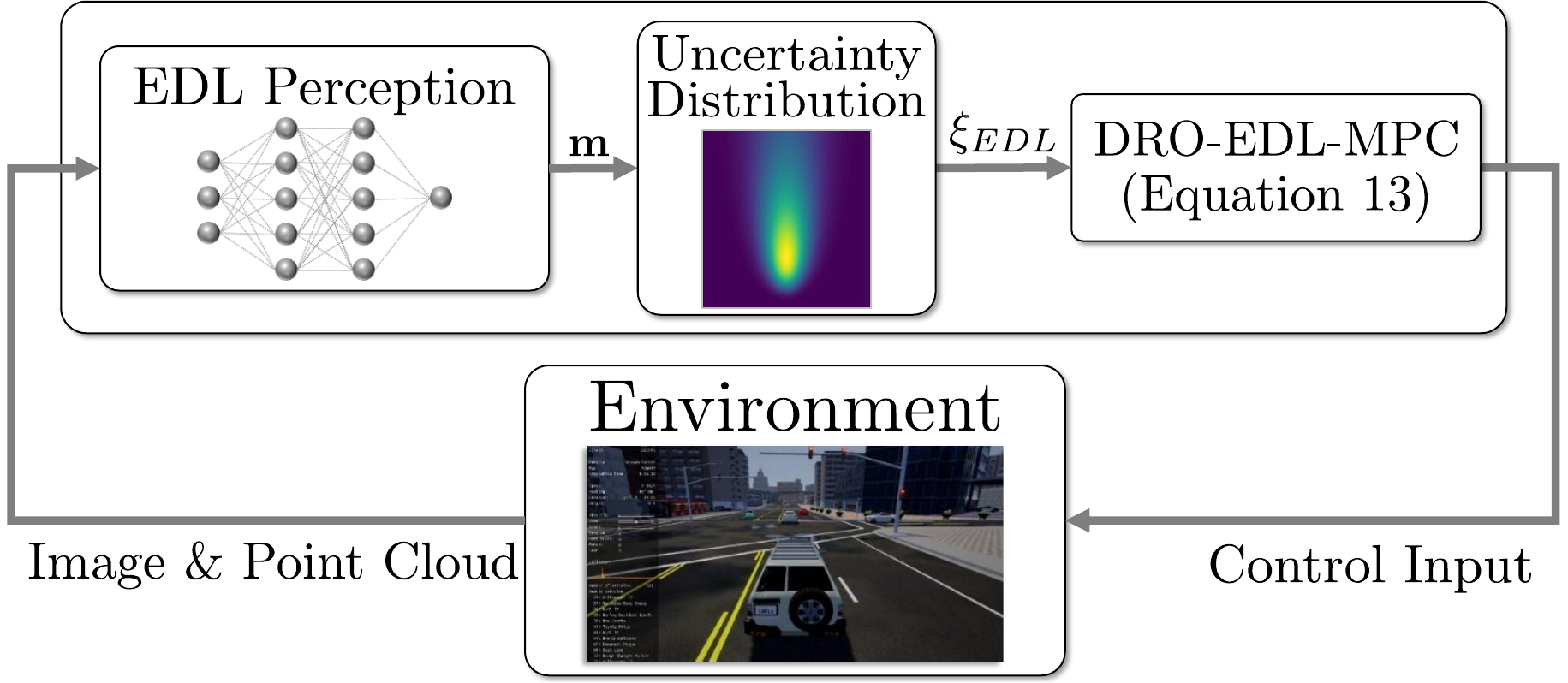}
\vspace*{-0.1in}
\caption{Illustration of the overall framework.}
\vspace*{-0.2in}
\label{fig:overall_framework}
\end{figure}

The goal of this paper is to compute the optimal motion of an ego vehicle navigating an environment with static obstacles while explicitly accounting for and adapting to perception uncertainties. The overall framework of our method is illustrated in Fig.~\ref{fig:overall_framework}.

The dynamics of the ego vehicle consists of the nominal dynamics model $f:\mathbb{R}^{n_x}\times \mathbb{R}^{n_u}\rightarrow \mathbb{R}^{n_x}$ and unknown dynamics $g: \mathbb{R}^{n_x}\times \mathbb{R}^{n_u}\rightarrow \mathbb{R}^{n_x}$, represented as 
\begin{equation}
    \begin{aligned}
        \mathbf{x}(t+1)=f(\mathbf{x}(t),\mathbf{u}(t)) + g(\mathbf{x}(t),\mathbf{u}(t)),
    \end{aligned}
    \label{eq:dynamics}
\end{equation}
where $\mathbf{x}(t)\in\mathbb{X}\subseteq \mathbb{R}^{n_x}$ and $\mathbf{u}(t)\in\mathbb{U}\subseteq\mathbb{R}^{n_u}$ are the state of the ego vehicle and the control input at timestep $t$. The nominal dynamics can be derived from physical principles of the system, while the unknown components can be learned from data using Gaussian Process Regression (GPR) \cite{hakobyan_distributionally_2023}.

To derive the unknown dynamics $g$, we collect input data $\mathbf{X}_t$ and target data $\mathbf{Y}_t$. Input data $\mathbf{X}_t=\{(\mathbf{x}(t-1), \mathbf{u}(t-1)), \ldots, (\mathbf{x}(t-M), \mathbf{u}(t-M))\}$ consists of the state and control input pairs of the past $M$ time steps. Target data $\mathbf{Y}_t=\{\Delta \mathbf{x}(t), ..., \Delta\mathbf{x}(t-M+1)\}$ consists of the residual $\Delta \mathbf{x}(t)$ between the nominal model $f(\mathbf{x}(t), \mathbf{u}(t))$ and the observed state $\mathbf{x}(t+1)$. Using GPR, we model the unknown dynamics by learning the mean $\mu^g(\mathbf{x}(t), \mathbf{u}(t))$ and covariance $\Sigma^g(\mathbf{x}(t), \mathbf{u}(t))$, characterizing the uncertainty as a Gaussian disturbance. 
Also, to propagate the uncertainty through nonlinear dynamics \eqref{eq:dynamics} over time, we use unscented transform (UT) \cite{ hakobyan_distributionally_2023}. More details can be found in \cite{hakobyan_distributionally_2023}.

The ego state $\mathbf{x}$ consists of the center positions $\mathbf{c}^x\in\mathbb{R}^{n_c}$, the heading angle $\phi^x\in \mathbb{R}$, and the speed $v^x \in\mathbb{R}$ of the ego vehicle. That is, $\mathbf{x}=(\mathbf{c}^x, \phi^x, v^x)$. For a 2-dimensional example, $n_c=2$ and $\mathbf{c}^x=(c_1^x, c_2^x)$ represents the coordinates of the center of the ego vehicle. Also, the ego has attributes for the lengths from the center to the boundary $\mathbf{b}^x\in\mathbb{R}^{n_c},$  where, for example, $\mathbf{b}^x=(b^x_1, b^x_2)$ represents the half of the width and length of the vehicle. The input of the ego vehicle consists of acceleration and steering angle.

The ego vehicle must avoid collisions with static obstacles in the environment.
Similar to the ego state $\mathbf{x}$, the obstacle state $\xi\in\mathbb{R}^{n_\xi}$ consists of the center positions $\mathbf{c}^\xi\in\mathbb{R}^{n_c}$, the heading angle $\phi^\xi\in \mathbb{R}$, and the speed $v^\xi\in \mathbb{R}$. That is, $\xi=(\mathbf{c}^\xi, \phi^\xi, v^\xi)$. The obstacle also has the attributes for the lengths from the center to the boundary $\mathbf{b}^\xi\in\mathbb{R}^{n_c}$. To prevent collisions between the ego vehicle and obstacles, we define a safety loss function $\ell: \mathbb{R}^{n_x}\times \mathbb{R}^{n_\xi}\rightarrow \mathbb{R}$ to quantify the level of safety. The safety constraint is encoded by 
\begin{equation}\label{eq:loss}
    \begin{aligned}
        &\ell(\mathbf{x}, \xi)=(r^x+r^\xi)^2-\|\mathbf{c}^x-\mathbf{c}^\xi\|^2_2\leq0,
    \end{aligned}
\end{equation}
 where $r^{\mathbf s}$ denotes the radius of any state $\mathbf{s}\in\{x, \mathbf{\xi}\}$ and is defined as $ r^{\mathbf s}:=\|\mathbf{b}^\mathbf{s}\|_2=\sqrt{\sum_{i}^{n_c}(b^\mathbf{s}_i)^2}$.

The obstacle state $\xi$ is estimated from sensor data $o$ using a perception model $F$ during the initialization stage. 
Since the perception model on robotic platforms typically employs a lightweight architecture and may encounter uncertain observations during deployment, it is essential to account for both data and model uncertainties.
To address data uncertainty, we employ Conditional Value-at-Risk (CVaR) constraint, $\text{CVaR}_\epsilon^{\mathbb{P}}[\ell(\mathbf{x}, \xi)]\leq 0$\footnote{
For loss $L\sim\mathbb{P}_L$, 
\ensuremath{
    \text{CVaR}_\epsilon^{\mathbb{P}_L}[L]:=\min_{z\in\mathbb{R}}\mathbb{E}^{\mathbb{P}_L}\left[z+\frac{(L-z)^+}{1-\epsilon}\right]} where $(L-z)^+=\max(L-z, 0)$. This is the average of the worst $(1-\epsilon)\%$ of outcomes.}, which enables the evaluation of expected loss in extreme scenarios. Here, $\mathbb{P}$ represents the probability distribution of $\xi$, and $\epsilon \in [0.5,1)$ is the confidence level parameter.
For model uncertainty, we use a distributionally robust approach. Given an ambiguity set $\mathbb{D}$ containing distribution $\mathbb{P}$, we formulate the distributionally robust CVaR constraint as $\max_{\mathbb{P}\in\mathbb{D}}\text{CVaR}_\epsilon^\mathbb{P}[\ell(\mathbf{x}, \mathbf{\xi})]\leq0.$ This constraint ensures safety by optimizing against the worst-case distribution within the ambiguity set.

The distributionally robust MPC problem with the stage-wise cost $c$ and the terminal cost $q$ is formulated as follows:
\begin{subequations}
\vspace*{-0.15in}
    \label{eq:mpc}
    \begin{align}
\min_{\mathbf{u}} \quad & \sum_{t=0}^{T-1}c(\mathbf{x}(t),\mathbf{u}(t))+q(\mathbf{x}(T))\\
\textrm{s.t.} \quad 
& \mathbf{x}(t)\in\mathbb{X}, \,\forall t\in \mathbb{Z}_{0:T},~~ \mathbf{u}(t)\in\mathbb{U},\,\forall t\in \mathbb{Z}_{0:T-1} \label{eq:MPC_state_action} \\
  & \eqref{eq:dynamics},~\forall t\in \mathbb{Z}_{0:T-1}  \label{eq:MPC_dynamics} \\
  &  \max_{\mathbb{P}\in\mathbb{D}}\text{CVaR}_\epsilon^\mathbb{P}[\ell(\mathbf{x}(t), \mathbf{\xi})]\leq0, ~\forall t\in \mathbb{Z}_{0:T}.  \label{eq:mpc_loss}
    \end{align}
\end{subequations}
Here, $\mathbb{Z}_{0:T}$ denotes the set of integers $\{0,1,\ldots, T\}$, and $T$ is the length of the planning horizon. The obstacle $\xi$ is assumed to be constant during the MPC prediction steps as it represents a static obstacle. For future research, we plan to extend this approach to handle dynamic obstacles by incorporating trajectory forecasting models \cite{marvi_evidential_2025}.
\section{EDL-based Safety Constraint}

In the distributionally robust formulation, one of the most critical components is how to construct the ambiguity set to balance robustness and performance. Existing DRO approaches construct these ambiguity sets using heuristically tuned parameters, such as Wasserstein distance \cite{hakobyan_distributionally_2023, safaoui_distributionally_2024} or moment constraints \cite{delage_distributionally_2010, wang_signal-devices_2025}. These approaches often fail to account for the true uncertainty inherent in the perception model, leading to overly conservative solutions even when the perception model expresses high confidence. To address this limitation, we employ an EDL perception model \cite{amini_deep_2020} to directly estimate the uncertainty distribution and utilize it to formulate more appropriate ambiguity sets. In our implementation, we consider the uncertainty distribution of $\mathbf{c}^\xi$ to formulate the ambiguity set, while using single predictions for the other terms $\mathbf{b}^\xi, \phi^\xi,$ and $v^\xi$ for simplicity.
\subsection{EDL-based Perception}
EDL assumes that the regression target $c_i^\xi$ is a random variable drawn from a Gaussian distribution $\mathcal{N}(\mu_i, \sigma_i^2)$ with the mean $\mu_i$ and variance $\sigma_i^2$.
Also, it assumes that the mean $\mu_i$ follows a Gaussian prior $\mathcal{N}(\gamma_i, \sigma_i^2/\lambda_i)$ with $\gamma_i\in\mathbb{R}, \lambda_i>0$ and the variance $\sigma_i^2$ follows the Inverse-Gamma prior with $\alpha_i>1, \beta_i>0$. Let us denote $\theta=(\mu_i,\sigma_i^2)$ and $m_i=(\gamma_i, \lambda_i, \alpha_i, \beta_i)$. Then the posterior $p(\theta|m_i)$ is defined as the Normal Inverse-Gamma (NIG) distribution $NIG(\theta|m_i)$, which is also called the evidential distribution \cite{amini_deep_2020}. The trained EDL model estimates the parameter $m_i$ of the evidential distribution, from which the prediction value of $c_i^\xi$, data uncertainty, and model uncertainty can be computed, respectively, as 
\begin{equation*}
    \mathbb{E}[\mu_i]=\gamma_i, ~~\mathbb{E}[\sigma_i^2]=\frac{\beta_i}{\alpha_i-1},~~ Var[\mu_i]=\frac{\beta_i}{\lambda_i(\alpha_i-1)}.
\end{equation*}
For an EDL model with the regression target $\mathbf{c}^\xi =(c_1^\xi,\ldots, c_{n_c}^\xi)$, it predicts the parameter $\mathbf{m}=(m_1, \ldots, m_{n_c})$. Because current EDL models, such as \cite{paat_medl-u_2024}, predict each $m_i$ separately at each head, we also assume each $c_i^\xi$ is independent and $\mathbf{c}^\xi$ is drawn from an $n_c$-dimensional multivariate Gaussian distribution $\mathcal{N}(\boldsymbol{\mu},\Sigma)$ where $\boldsymbol{\mu}=(\mu_1, \ldots, \mu_{n_c})$ and $\Sigma$ is a diagonal matrix of $(\sigma_1^2, \ldots, \sigma_{n_c}^2)$. 

\subsection{EDL-based Ambiguity Set and Safety Constraint}
The evidential distribution can be interpreted as a higher-order distribution that captures model uncertainty, and a lower-order realization of the evidential distribution, which is a normal distribution, represents data uncertainty. We leverage this evidential distribution to construct our ambiguity set and formulate distributionally robust safety constraints, enabling us to account for both model and data uncertainties.

\begin{definition}\label{definition:ambiguity}
    \textnormal{(EDL ambiguity set).} Given a cumulative probability threshold $\eta_i\in \mathbb{R}$ and evidential distribution parameter $m_i$, the ambiguity set for axis $i$ is defined as
    \begin{equation*}\label{eq:amb_set_i}
    \begin{aligned}
    \mathbb{D}_i(\eta_i|m_i):=\left\{\mathcal{N}(\mu, \sigma^2)| \int_{\theta=(\mu, \sigma^2)}NIG(\theta | m_i)d\theta= \eta_i\right\}.
    \end{aligned}
    \end{equation*}
    For all axis, given $\eta\in\mathbb{R}$ and $\mathbf{m}=(m_1,\ldots, m_{n_c})$, the ambiguity set $\mathbb{D}({\eta}|\mathbf{m})$ is defined as 
    \begin{equation*}\label{eq:amb_set}
    \begin{aligned}
        \mathbb{D}({\eta}|\mathbf{m}):=\{\mathcal{N}(\boldsymbol{\mu}, \Sigma)|
        &\mathcal{N}(\mu_i, \sigma_i^2)\in\mathbb{D}_i(\sqrt[n_c]{\eta}|m_i), \forall i\}.
    \end{aligned}
    \end{equation*}
\end{definition}
\noindent
This means that the ambiguity set for axis $i$ is a set of Gaussian distributions whose parameters lie within the $\eta_i$-confidence region of the evidential distribution.
Furthermore, under the assumption that the target follows a Gaussian distribution and EDL model is well-trained, 
the true target distribution lies within this ambiguity set with confidence $\eta_i$.
EDL generates a dispersed evidential distribution under high model uncertainty and a concentrated one when uncertainty is low, causing our ambiguity set to expand or contract accordingly. This ensures that even with a fixed confidence level $\eta_i$, our approach automatically adapts its conservativeness based on the perception model's uncertainty.

Next, we formulate the distributionally robust safety loss within this ambiguity set. 
\begin{definition}\label{definition:DR-EDL-CVaR}
    \textnormal{(DR-EDL-CVaR).} The distributionally robust safety loss given EDL ambiguity set $\mathbb{D}(\eta|\mathbf{m})$ is 
    \begin{equation}\label{eq:DR-EDL-CVaR}
    \begin{aligned}
    \text{DR-EDL-CVaR}_\epsilon^{\mathbb{D}(\eta|\mathbf{m})}[\ell(\mathbf{x},\xi)]
    :=\max_{{\mathbb{P}\in\mathbb{D}(\eta|\mathbf{m})}}\text{CVaR}_\epsilon^\mathbb{P}[\ell(\mathbf{x}, \xi)].
    \end{aligned}
    \end{equation}
\end{definition}
\noindent
The safety loss $\ell$ is defined in \eqref{eq:loss}.
This formulation enhances the robustness of our control approach by simultaneously addressing both types of uncertainty: data uncertainty through CVaR risk metric at confidence level $\epsilon$, and model uncertainty by optimizing against the worst-case CVaR within the ambiguity set at confidence level $\eta$. 

We assume that the target distributions follow independent Gaussian, which may not always hold in complex real-world environments. 
Multi-modal uncertainty can be addressed by representing each mode as an individual obstacle. Our framework supports multi-obstacle collision avoidance, and multi-modal EDL can be implemented following \cite{marvi_evidential_2025}.
To handle correlated Gaussian distributions, the ambiguity set formulation can be extended using the Normal Inverse-Wishart distribution, which is the output of a multivariate EDL model \cite{meinert_multivariate_2022}. In addition, model misspecification can be addressed by adopting a risk metric defined over a family of distributions, as discussed in \cite{yu_general_2009}, Proposition 2.

\section{Solution}

In this section, we present a tractable approach to obtain a solution that satisfies the safety constraint with distributionally robust safety loss~\eqref{eq:DR-EDL-CVaR}. 
This newly proposed safety constraint is integrated into the MPC framework.

\subsection{Identification of the worst-case obstacle distribution} \label{subsec:4.1}
Because it is nontrivial to find the closed-form solution of the ambiguity set, we define a \textit{surrogate ambiguity set}, which facilitates calculating the worst-case loss.

\begin{definition}\textnormal{(Surrogate ambiguity set).}  Let $\Theta:=\{\theta:NIG(\theta|m_i)\geq c_{th}$\} denote a superlevel set that satisfies $\int_{\Theta}NIG(\theta|m_i)d\theta=\eta_i$. Let $\mu_{i,min}, \mu_{i,max}$ and $\sigma^2_{i,min}, \sigma^2_{i,max}$ be the extreme values of $\mu$ and $\sigma^2$, respectively, in $\Theta$. Let $\mathcal{I}_{i,\mu}:=[\mu_{i,min},\mu_{i,max}]$ and $\mathcal{I}_{i,\sigma^2}:=[\sigma^2_{i,min},\sigma^2_{i,max}]$.
We define the surrogate ambiguity set as 
    \begin{multline}
    \mathbb{D}^{sur}_{i}(\eta_i|m_i):=\{\mathcal{N}(\mu, \sigma^2): \mu\in\mathcal{I}_{i,\mu}, \sigma^2\in\mathcal{I}_{i,\sigma^2}\}.
    \label{eq:sur_amb_set_i}
    \end{multline}
\label{def:sur_amb_set}
\end{definition}
This means that the surrogate ambiguity set is a set of Gaussian distributions whose $\mu$ and $\sigma^2$ lie in a rectangular region $\mathcal{I}_{i,\mu} \times \mathcal{I}_{i, \sigma^2}$.
The rectangular region $\mathcal{I}_{i,\mu} \times  \mathcal{I}_{i,\sigma^2}$ encloses the contour $\Theta$, resulting in for any $\eta_i$ and $m_i$
\begin{equation}\label{eq:D_surrogate_superset}
    \begin{aligned}
        \mathbb{D}_i(\eta_i|m_i)\subset\mathbb{D}_i^{sur}(\eta_i|m_i).
    \end{aligned}
\end{equation}
It also implies that the cumulative probability of the NIG distribution at $\mathcal{I}_{i,\mu}\times\mathcal{I}_{i,\sigma^2}$ is larger than $\eta_i$, which makes the ambiguity set more conservative.
In Fig. \ref{fig:ambiguity_set}, the contour $ \Theta$ that defines the original ambiguity set and $\mathcal{I}_{i,\mu}\times\mathcal{I}_{i,\sigma^2}$ that defines the surrogate ambiguity set are illustrated as a solid round and a dashed rectangle, respectively.

\begin{proposition}\label{prop:sur_amb_set_solution}
\textnormal{(Conservative surrogate ambiguity set solution).}
Given the surrogate ambiguity set $\mathbb{D}_i^{sur}(\eta_i|m_i)$, the distributionally robust safety loss satisfies
\begin{equation*}
        \begin{aligned}
            &\max_{{\mathbb{P}\in\mathbb{D}(\eta|\mathbf{m})}}\text{CVaR}_\epsilon^\mathbb{P}[\ell(\mathbf{x}(t), \xi)] \\
            &\leq(r^x+r^\xi)^2-\sum_i^{n_c}\max_{\mathbb{P}_i\in\mathbb{D}^{sur}_i(\eta_i|m_i)}\text{CVaR}_{\epsilon}^{\mathbb{P}_i}[-|c_i^x-c_i^\xi|]^2.
        \end{aligned}
    \end{equation*}
    Furthermore, 
    \begin{itemize}
        \item If $c_i^x\in\mathcal{I}_{i,\mu}$,
    \begin{equation*}
        \begin{aligned}
            \max_{\mathbb{P}_i\in\mathbb{D}_i^{sur}(\eta_i|m_i)}\text{CVaR}_{\epsilon}^{\mathbb{P}_i}[-|c_i^x-c_i^\xi|]=\kappa\cdot\sigma_{i,min}<0,
        \end{aligned}
    \end{equation*}
    where $\kappa=\frac{1}{1-\epsilon}
\sqrt{\frac{2}{\pi}}(\exp{(-[erf^{-1}(\epsilon-1)]^2)}-1)$ and $erf^{-1}(\cdot)$ the inverse of the error function.
    \item If $c_i^x\in\mathcal{I}_{i,unsafe}\setminus \mathcal{I}_{i,\mu}$,
    \begin{equation*}
\max_{\mathbb{P}_i\in\mathbb{D}^{sur}_i(\eta_i|m_i)}\text{CVaR}_{\epsilon}^{\mathbb{P}_i}[-|c_i^x-c_i^\xi|] \leq\kappa\cdot\sigma_{i,min}<0,
    \end{equation*}
    where $\mathcal{I}_{i,unsafe}=[\mu_{i,min}-\delta\cdot\sigma_{i,max}, \mu_{i,max}+\delta\cdot\sigma_{i,max}]$ for $\delta=\frac{1}{1-\epsilon}\frac{1}{\sqrt{2\pi}}\exp{(-[erf^{-1}(2\epsilon-1)]^2)}>0$.
    \item If $c_i^x\notin \mathcal{I}_{i,unsafe},$
    \begin{equation*}
        \begin{aligned}
&\max_{\mathbb{P}_i\in\mathbb{D}^{sur}_i(\eta_i|m_i)}\text{CVaR}_{\epsilon}^{\mathbb{P}_i}[-|c_i^x-c_i^\xi|] \\
            &<
            -|c_i^x-\gamma_i|+\frac{\mu_{i,max}-\mu_{i,min}}{2}+\delta\cdot\sigma_{i,max}<0,
        \end{aligned}
    \end{equation*}
    where $\gamma_i$ is given in the prediction result $m_i$.
    \end{itemize}
\end{proposition}
\begin{figure}[t]
\centering
\includegraphics[width=0.8\linewidth]{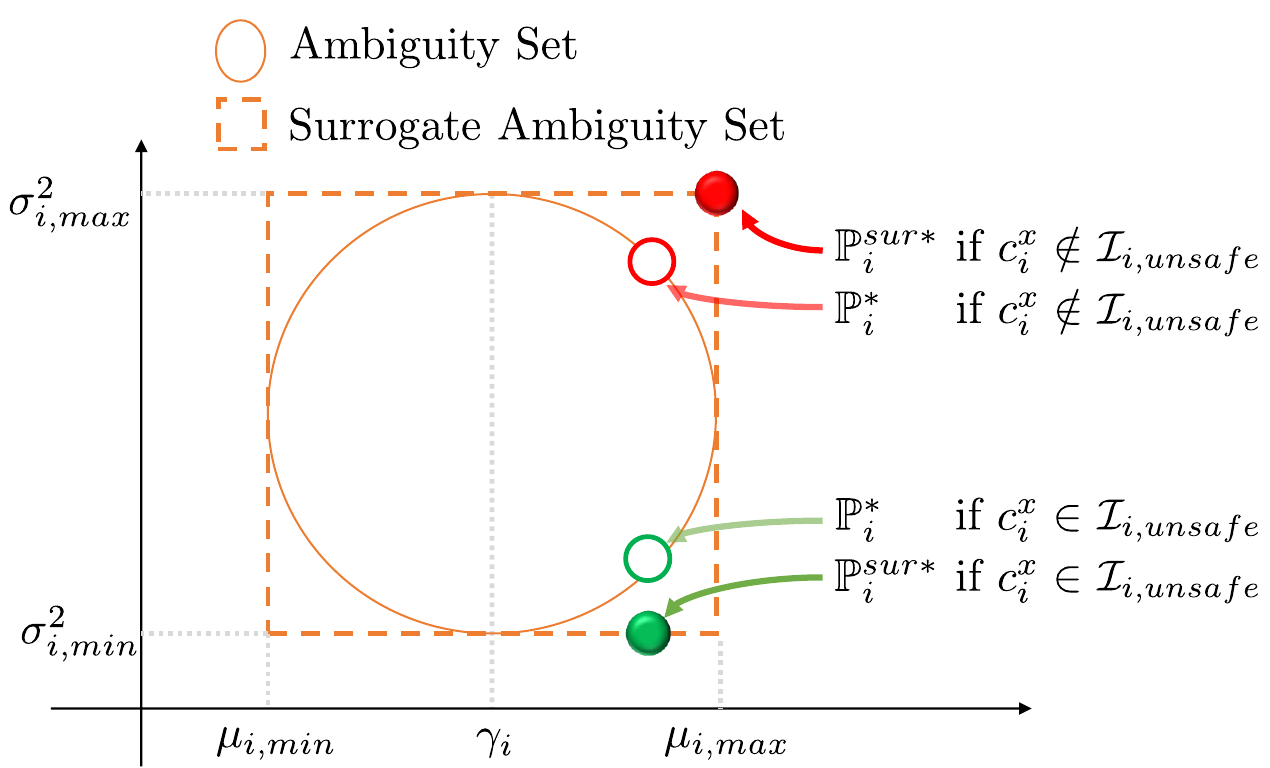}
\vspace*{-0.1in}
\caption{Illustration of moment sets $(\mu_i, \sigma_i^2)$ for the original and surrogate ambiguity sets along axis $i$. The worst-case distributions within the original and surrogate ambiguity sets are identified as $\mathbb{P}_i^*$ and $\mathbb{P}_i^{sur*}$, respectively.}
\label{fig:ambiguity_set}
\vspace*{-0.2in}
\end{figure}

\begin{proof} The worst-case CVaR of the safety loss  \eqref{eq:loss} is
\begin{align}
    &\max_{{\mathbb{P}\in\mathbb{D}(\eta|\mathbf{m})}}\text{CVaR}_\epsilon^\mathbb{P}[\ell(\mathbf{x}(t), \xi)] \notag \\
        &=(r^x+r^\xi)^2+\max_{\mathbb{P}\in\mathbb{D}(\eta|\mathbf{m})}\text{CVaR}_\epsilon^\mathbb{P}[-\sum_i^{n_c}|c_i^x-c_i^\xi|^2] \notag \\
        &\leq(r^x+r^\xi)^2-\sum_i^{n_c}\max_{\mathbb{P}_i\in\mathbb{D}^{sur}_i(\eta_i|m_i)}\text{CVaR}_{\epsilon}^{\mathbb{P}_i}[-|c_i^x-c_i^\xi|]^2.\label{eq:14}
    \end{align}
The equality holds as the axes of the obstacle distribution are independent of each other and by Lemma \ref{lemma:CVaR}(a) and \eqref{eq:D_surrogate_superset}.

We will derive the upper bound of $\max_{\mathbb{P}_i\in\mathbb{D}^{sur}_i(\eta_i|m_i)}\text{CVaR}_{\epsilon}^{\mathbb{P}_i}[-|c_i^x-c_i^\xi|]$ that is strictly negative and thus minimizes the squared term in \eqref{eq:14}.
The difference $X_d=c_i^x-c_i^\xi$ follows the normal distribution $\mathcal{N}(c_i^x-\mu_i,\sigma_i^2)$ when $c_i^\xi$ follows $\mathbb{P}_i=\mathcal{N}(\mu_i, \sigma_i^2)$. We denote $\mu_d:=c_i^x-\mu_i$ and $\sigma_d^2:=\sigma_i^2$. The distance $|X_d|$ follows a folded normal distribution. By Lemma \ref{lemma:CVaR}(b), 
$\text{CVaR}_{\epsilon}^{\mathbb{P}_i}[-|X_d|]$, referred to as the \textit{negative distance CVaR}, is a monotonically decreasing function with respect to $|\mu_d|$, and thus, its maximum is attained when $|\mu_d|$ is minimized.

If $c_i^x\in \mathcal{I}_{i,\mu}$, $|\mu_d|=0$ maximizes the negative distance CVaR. The distribution of $-|X_d|$ is a half-normal distribution, and the negative distance CVaR is $\kappa\cdot\sigma_i$ by Lemma~\ref{lemma:CVaR}(c). Therefore, the distribution $\mathbb{P}_i=\mathcal{N}(c_i^x,\sigma_{i,min}^2)$, which lies in the surrogate ambiguity set $\mathbb{D}_i^{sur}(\eta_i|m_i)$, yields the maximum negative distance CVaR of $\kappa\cdot\sigma_{i,min}$.

If $c_i^x\in\mathcal{I}_{i,unsafe}\setminus \mathcal{I}_{i,\mu}$, the mean $|\mu_d|$ of $|X_d|$ is larger than 0 because the mean $\mu_i$ of $c_i^\xi$ is confined in $\mathcal{I}_{i,\mu}$. 
Because the maximum negative distance CVaR decreases with respect to $|\mu_d|$ by Lemma~\ref{lemma:CVaR}(b), it is no larger than $\kappa\cdot \sigma_{i,min}$. 

If $c_i^x\notin\mathcal{I}_{i,unsafe}$, when $\mu_d > 0$, the $\epsilon$-quantile of the distribution of $X_d$, i.e., $\text{VaR}_{\epsilon}^{\mathbb{P}_i}[X_d]$, is a positive value by construction of $\mathcal{I}_{i,unsafe}$, in particular the definition of $\delta$. Similarly, when $\mu_d<0$, $\text{VaR}_{\epsilon}^{\mathbb{P}_i}[X_d]< 0$ by construction of $\mathcal{I}_{i,unsafe}$. Under this condition, Lemma~\ref{lemma:CVaR}(d) tells that $\max_{\mathbb{P}_i} -|\mu_d|+\delta \cdot \sigma_d > \max_{\mathbb{P}_i} \text{CVaR}_{\epsilon}^{\mathbb{P}_i}[-|X_d|]$. The maximum of $-|\mu_d|+\delta\cdot\sigma_d$ is attained at $\sigma_d=\sigma_{i,max}$ because $\delta>0$ and $|\mu_d|=|c_i^x-\mu_{i,min}|$ when $c_i^x < \mu_{i,min}-\delta\cdot\sigma_{i,max}$ or at $\sigma_d=\sigma_{i,max}$ and $|\mu_d|=|c_i^x-\mu_{i,max}|$ when $c_i^x>\mu_{i,\max}+\delta \cdot\sigma_{i,max}$. Using $\gamma_i$, which is the center of $[\mu_{i,min}, \mu_{i,max}]$, we can represent such optimal $|\mu_d|$ as $|c_i^x-\gamma_i|-(\mu_{i,max}-\mu_{i,min})/2$. Therefore, $-|c_i^x-\gamma_i|+(\mu_{i,max}-\mu_{i,min})/{2}+\delta\cdot\sigma_{i,max}> \max_{\mathbb{P}_i} \text{CVaR}_{\epsilon}^{\mathbb{P}_i}[-|X_d|]$ and this upper bound is negative because $c_i^x\notin \mathcal{I}_{i,unsafe}$ and by construction of $\mathcal{I}_{i,unsafe}$.
\end{proof}

\begin{lemma} \label{lemma:CVaR}
\textnormal{(Properties of the Conditional Value-at-Risk).}
Given a random variable $X\sim\mathbb{P}=\mathcal{N}(\mu, \sigma^2)$ and a threshold $\epsilon\in [0,1)$, the following holds. 
\begin{enumerate}[label=(\alph*)]
    \item $\text{CVaR}_{\epsilon}^\mathbb{P}[-|X|^2]\leq -\text{CVaR}_{\epsilon}^\mathbb{P}[-|X|]^2$
    \item For a fixed $\sigma^2$, $\text{CVaR}_{\epsilon}^\mathbb{P}[-|X|]$ is monotonically decreasing with respect to $|\mu|$.
    \item If $\mu=0$, then $\text{CVaR}_\epsilon^\mathbb{P}[-|X|]=\kappa\cdot\sigma$.
    \item If $\text{VaR}_\epsilon^\mathbb{P}[X]> 0 \land \mu> 0$, or $\text{VaR}_\epsilon^\mathbb{P}[X]< 0 \land \mu< 0$, then $-|\mu|+\delta\cdot\sigma> \text{CVaR}_\epsilon^\mathbb{P}[-|X|]$.
\end{enumerate}
\end{lemma}

\ifarxiv
The proof of this lemma can be found in the Appendix.
\else
The proof of this lemma can be found in the Appendix of the extended version \cite{ham_dro-edl-mpc_2025}.
\fi 

According to Proposition~\ref{prop:sur_amb_set_solution}, the worst-case obstacle distribution varies depending on the ego vehicle position $c_i^x$. A one-dimensional example is illustrated in Fig. \ref{fig:ambiguity_set}. When the ego vehicle is distant from the obstacle $(c_i^x > \mu_{i,max}+\delta\cdot \sigma_{i,max})$, (an upper bound of) the optimal negative distance CVaR is attained at $\mathbb{P}_i^{sur*}=\mathcal{N}(\mu_{i,max}, \sigma_{i,max}^2)$ (filled red circle) within the surrogate ambiguity set. Within the original ambiguity set, $\mathbb{P}_i^*$ (empty red circle) denotes the worst-case obstacle distribution. In contrast, when the ego vehicle is close to the obstacle $(c_i^x\in\mathcal{I}_{i,unsafe})$, (an upper bound of) the maximum negative distance CVaR is attained at $\mathbb{P}_i^{sur*}=\mathcal{N}(c_i^x, \sigma_{i,min}^2)$ (filled green circle) within the surrogate ambiguity set. Within the original ambiguity set, $\mathbb{P}_i^*$ (empty green circle) denotes the worst-case distribution.

Proposition~\ref{prop:sur_amb_set_solution} describes how to identify the worst-case distribution $\mathbb{P}_i^{sur*}$ within the surrogate ambiguity set. Thus, we can directly find the worst-case obstacle distribution instead of focusing on the distribution of the loss function. 

However, calculating the extrema $\mu_{i,min}, \mu_{i,max}, \sigma_{i,min}$ and $ \sigma_{i,max}$  given the confidence threshold $\eta_i$ remains computationally challenging, as it requires complex numerical optimization problems. To mitigate this issue, we employ standardized solutions.
The NIG distribution consists of the normal distribution and the inverse gamma distribution. These distributions have location parameter $\gamma_i$ and scale parameter $\sigma_i, \beta_i$. Therefore, we can transpose $\text{NIG}(\mu_i,\sigma_i^2|\gamma_i,\lambda_i,\alpha_i,\beta_i)$ into $\text{NIG}(\mu_{z,i},\sigma_{z,i}^2|0,1,\alpha_i,1)$ using following relation:
\begin{equation}
    \begin{aligned}
    \mu_i = \mu_{z,i}\cdot\sqrt{\frac{\beta_i}{\lambda_i}}+\gamma_i, && \sigma_i=\sigma_{z,i}\cdot\sqrt{\beta_i}.
    \end{aligned}
    \label{eq:standardize}
\end{equation}
The shape parameter $\alpha_i$ remains unchanged as it is an intrinsic property of the distribution.
The contour of the transposed NIG distribution at confidence $\eta_i$ is calculated using numerical integration \cite{philip_methods_2007} and the extrema $(\mu_{z,i,min}, \mu_{z, i, max}, \sigma_{z, i, min}^2, \sigma_{z, i, max}^2)$ are calculated using the Brent's method \cite{richard_algorithms_2013}, which is a numerical root finding method. These extrema are computed offline and stored in the lookup table $\mathcal{M}_\alpha$ for different $\alpha\in [1.01, 10.00]$. At execution time, the perception model estimates the NIG distribution parameters $m_i$, and the extrema are queried from $\mathcal{M}_\alpha$ without computation overhead.
These extrema are transformed back to the extrema of the original distribution using (\ref{eq:standardize}). Consequently, the standardized approach and Proposition~\ref{prop:sur_amb_set_solution} enables determining the worst-case obstacle distribution $\mathbb{P}_i^{sur*}$ in a computationally tractable manner.

\subsection{Conservative constraint for tractable MPC} \label{subsec:4.2}
\begin{figure}[t]
\centering
\includegraphics[width=0.9\linewidth]{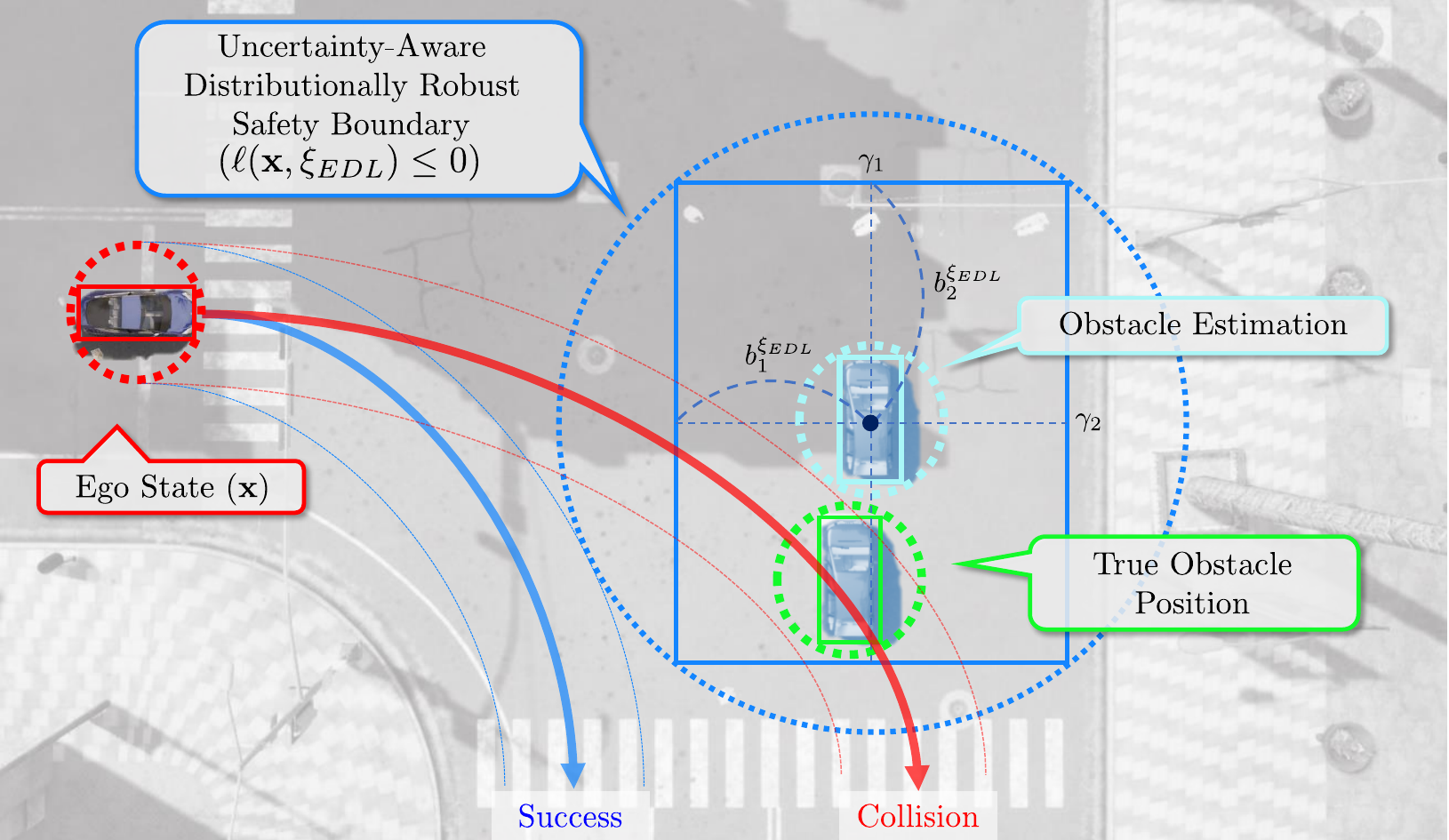}
\vspace*{-0.1in}
\caption{Estimated obstacle position (cyan circle) may differ from the true obstacle position (green circle). By introducing $\xi_{EDL}$ (large blue rectangle), we enforce the constraint $\ell(\mathbf{x}, \xi_{EDL})\leq 0$ (large blue circle), thereby satisfying the distributionally robust safety constraint.}
\label{fig:worst-case}
\vspace*{-0.2in}
\end{figure}

The safety constraint can be formulated using the worst-case obstacle distribution derived by Proposition~\ref{prop:sur_amb_set_solution}. However, these constraints vary depending on the ego position $c_i^x$, for example, whether $c_i^x\in \mathcal{I}_{i,\mu}$, and thus translate into a mixed-integer programming formulation, which can be solved using the Big-M method \cite{schouwenaars_mixed_2001}. To avoid the Big-M approach, which can be computationally expensive and prone to numerical instability, we introduce a conservative reformulation of the worst-case CVaR that is continuous with the ego state, eliminating the need for integer variables.

\begin{proposition}\label{prop:xi_edl}
    \textnormal{(Constraint reformulation).} Given the perception results $\mathbf{m}, \phi^\xi, v^\xi$, and $\mathbf{b}^\xi$, define    $\xi_{EDL} =(\mathbf{c}^{\xi_{EDL}}, \phi^{\xi_{EDL}}, v^{\xi_{EDL}})$ as
    \begin{equation*}
c_i^{\xi_{EDL}}:=\gamma_i,~~\phi^{\xi_{EDL}}:=\phi^\xi,~~v^{\xi_{EDL}}:=v^\xi,
    \end{equation*}
    and its attribute $\mathbf{b}^{\xi_{EDL}}$ as
    \begin{equation*}
        b_i^{\xi_{EDL}}:=(\mu_{i,max}-\mu_{i,min})/2+\delta\cdot\sigma_{i,max}+r^\xi,
    \end{equation*}
    where $\delta$ is provided in Proposition~\ref{prop:sur_amb_set_solution}. Then, for any $\mathbf{x}$,
    \begin{equation*}\label{eq:DR-EDL-CVaR_trac}
        \begin{aligned}
        \ell(\mathbf{x}, \xi_{EDL})\leq 0 \implies \text{DR-EDL-CVaR}_\epsilon^{\mathbb{D}(\eta|\mathbf{m})}[\ell(\mathbf{x},\xi)]\leq 0.
        \end{aligned}
    \end{equation*}
\end{proposition}
\begin{proof}
Suppose $c_i^x\in \mathcal{I}_{i,unsafe}$ for $i=1,\ldots, n_{in}$ and $c_i^x\notin \mathcal{I}_{i,unsafe}$ for $i\in n_{in}+1, \ldots, n_c$. By Proposition~\ref{prop:sur_amb_set_solution},  
\begin{equation*}
    \begin{aligned}
        &\max_{\mathbb{P}\in\mathbb{D}(\eta|\mathbf{m})}\text{CVaR}_\epsilon^\mathbb{P}[\ell(\mathbf{x},\xi)]\\
        \leq &(r^x+r^\xi)^2-\sum_{i=1}^{n_{in}}(\kappa\cdot\sigma_{i,min})^2-\sum_{i={n_{in}+1}}^{n_{c}}(|c_i^x-\gamma_i|-h_i)^2,
            \end{aligned}
\end{equation*}
where $h_i=(\mu_{i,max}-\mu_{i,min})/2+\delta\cdot\sigma_{i,max}$. We denote the right-hand side by $f(\mathbf{x})$. 

Let us define a new obstacle state $\tilde{\xi}=(c^{\tilde{\xi}}, \phi^{\tilde{\xi}}, v^{\tilde{\xi}})$ with $c^{\tilde{\xi}}=\gamma_i,\phi^{\tilde{\xi}}=\phi^{\xi}, v^{\tilde{\xi}}=v^\xi, ||b_i^{\tilde{\xi}}||_2^2=r^\xi+\sqrt{\sum_i^{n_c}h_i^2}$, and let 
\begin{equation*}
    \begin{aligned} &g(\mathbf{x}):=\ell(\mathbf{x},\tilde{\xi}) =(r^x+r^\xi+\textstyle\sqrt{\sum_i^{n_c}h_i^2})^2-\sum_i^{n_c}(|c_i^x-\gamma_i|)^2.
    \end{aligned}
\end{equation*}
Note that $\tilde{\xi}$ is the same as $\xi_{EDL}$ except that the attribute ${b}_i^{\xi_{EDL}}=r^\xi+h_i$ incorporates the uncertainty margin $h_i$ directly into the size of the obstacle.

Given feasible sets $F:=\{\mathbf{x}|f(\mathbf{x})\leq 0\}$ and $G:=\{\mathbf{x}|g(\mathbf{x})\leq 0\}$, we aim to show $G\subseteq F$.
Let $\vec{v}(\mathbf{x}):=(|c_1^x-\gamma_1|, \ldots, |c_{n_c}^x-\gamma_{n_c}|), \vec{h}:=(h_1, \ldots, h_{n_c})$ and $s(\mathbf{x}):=(\min(h_1, |c_1^x-\gamma_1|), ..., \min(h_{n_c}, |c_{n_c}^x-\gamma_{n_c}|))$. Then,
\begin{align}
        \|\vec{v}(\mathbf{x})\|_2 &\leq \|\vec{s}(\mathbf{x})\|_2+\|\vec{v}(\mathbf{x})-\vec{s}(\mathbf{x})\|_2 \notag \\
        & \leq \|\vec{h}\|_2 + \textstyle\sqrt{\sum_{i={n_{in}+1}}^{n_{c}}{(|c_i^x-\gamma_i|-h_i)^2}}.\label{eq:v_s_inequality}
\end{align}
The first inequality is by the triangle inequality, and the second inequality is because 
$\|\vec{s}(\mathbf{x})\|_2\leq \|\vec{h}\|_2$ by definition and $s_i(\mathbf{x}) = |c_i^x-\gamma_i|$ for $i=1,\ldots, n_{in}$ and $h_i$ for $i=n_{in}+1, \ldots, n_c$.

Now, if $f(\mathbf{x})> 0$, we have 
$(r^x+r^\xi)^2-\sum_{i=n_{in}+1}^{n_{c}}(|c_i^x-\gamma_i|-h_i)^2 >0$ because $(\kappa\cdot \sigma_{i,min})^2\geq 0$. Therefore, \eqref{eq:v_s_inequality} becomes $\|\vec{v}(\mathbf{x})\|_2 < \|\vec{h}\|_2 + (r^x+r^\xi)$, and this leads to $g(\mathbf{x})>0$. This showed $F^c\subseteq G^c$ and therefore $G\subseteq F$. This implies that $g(\mathbf{x})\leq 0$ is more conservative than the original constraint $f(\mathbf{x})\leq 0$. 

To facilitate a simpler and more tractable representation of the obstacle attribute, we define the obstacle $\xi_{EDL}$ with $b_i^{\xi_{EDL}}=r^\xi+h_i$. This makes the radius $r^{\xi_{EDL}}=\|r^\xi+\vec{h}\|_2\geq r^\xi + \|\vec{h}\|_2$ and thus $\ell(\mathbf{x}, \xi_{EDL}) \geq \ell(\mathbf{x}, \tilde{\xi})$. Therefore, $\ell(\mathbf{x}, \xi_{EDL})\leq 0 \implies \ell(\mathbf{x},\tilde{\xi})\leq 0 \implies f(\mathbf{x})\leq 0 \implies \text{DR-EDL-CVaR}_\epsilon^{\mathbb{D}(\eta|\mathbf{m})}[\ell(\mathbf{x},\xi)]\leq 0.$
\end{proof}

This proposition defines a new obstacle, $\xi_{EDL}$, which provides a more conservative constraint that is continuous with respect to $c_i^x$. 
Fig. \ref{fig:worst-case} illustrates the perceived obstacle position, true obstacle position and the use of $\xi_{EDL}$.

\subsection{DRO-EDL-MPC Algorithm} \label{subsec:4.3}
The tractable constraint derived from Proposition \ref{prop:xi_edl} is integrated into the MPC framework. 
\begin{theorem}
    \textnormal{(DRO-EDL-MPC algorithm).} The solution of the following MPC problem is also a feasible solution to the original MPC problem \eqref{eq:mpc}.
    \begin{subequations}
    \begin{align}
\min_{\mathbf{u}} \quad & \sum_{t=0}^{T-1}c(\mathbf{x}(t),\mathbf{u}(t))+q(\mathbf{x}(T))\\
\textrm{s.t.} \quad 
  & \eqref{eq:MPC_state_action} - \eqref{eq:MPC_dynamics} \\ 
  & \ell(\mathbf{x}(t),\xi_{EDL})\leq 0, \ \forall t\in \mathbb{Z}_{0:T}\label{eq:thm1_loss}
    \end{align}
    \label{eq:mpc_trac}
\end{subequations}
\end{theorem}

\vspace*{-0.25in}
\begin{proof}
    The safety constraint \eqref{eq:thm1_loss} results in $\text{DR-EDL-CVaR}_\epsilon^{\mathbb{D}(\eta|\mathbf{m})}[\ell(\mathbf{x},\xi)]\leq 0$ by Proposition \ref{prop:xi_edl}. Therefore, the optimal solution of \eqref{eq:mpc_trac} satisfies the distributionally robust safety constraint \eqref{eq:mpc_loss}.
\end{proof}

The overall scheme is described in Algorithm 1. 
The standardized solutions for the given confidence threshold $\eta_i$ are precomputed offline, and stored in the memory $\mathcal{M}_{\alpha}$ (line 1).
During execution, the ego vehicle initially receives the sensor observations $o$ and predicts the NIG parameters $\mathbf{m}$ using the perception model (line 2).
The GPR-based dynamics model is trained from the collected state and control input pairs (line 3). 
The distributionally robust constraint is computed using the NIG parameters $\mathbf{m}$ and the precomputed memory $\mathcal{M}_{\alpha}$ (line 5). 
We solve the MPC problem~\eqref{eq:mpc_trac} with this distributionally robust constraint (line 6), and the first component of the optimal control input is applied to the environment (line 7). 
The next ego state is determined, and the algorithm repeats the same procedure.

\begin{algorithm}
\caption{DRO-EDL-MPC}\label{algo:algo1}
\begin{algorithmic}[1]
\Require Pretrained perception model weight $w$, confidence threshold $\eta$, CVaR probability $\epsilon$
\State Construct lookup table $\mathcal{M}_{\alpha},~~\forall\alpha\in[1.01,10.00]$
\State Predict $\mathbf{m}=\mathcal{F}(o|w)$
\State Collect dynamics observations $\mathbf{X}_t, \mathbf{Y}_t$ and train GPR.
\For{$t = 0,1, ...$}
    \State Compute $\xi_{EDL}$ in \eqref{eq:thm1_loss} using $\mathbf{m}$ and $\mathcal{M}_{\alpha}(\eta_i)$
    \State Solve \eqref{eq:mpc_trac} and get the optimal solution $\mathbf{u}^*$
    \State Apply $\mathbf{u}^*(0)$ to the environment
\EndFor
\end{algorithmic}
\end{algorithm}

\section{Experiments}

In this section, we validate the DRO-EDL-MPC algorithm (Algorithm~\ref{algo:algo1}) in the CARLA simulator and demonstrate its less conservative behavior under confident perception and more conservative behavior under uncertain environment.

\subsection{Experiment Settings}

The objective of the ego vehicle is to reach its destination while avoiding collisions with obstacles. The ego vehicle is equipped with a LiDAR and camera and uses the MEDL-U algorithm \cite{paat_medl-u_2024} as the EDL model. It takes LiDAR point clouds, camera images, and 2D bounding box predictions to predict 3D bounding boxes with uncertainty, represented as NIG distributions. We use YOLOv8 model \cite{glenn_ultralytics_2023} trained on COCO \cite{lin_microsoft_2014} to predict 2D bounding box and MEDL-U trained on the KITTI dataset \cite{geiger_vision_2013} using only the car class.

To compare confident and uncertain perception scenarios, we conduct experiments with different obstacle classes and sensor configurations. In the confident case, a car is used as the obstacle, and the LiDAR is mounted on the ego vehicle at the same height as in the training setup, ensuring that the input data remains in-distribution. In contrast, in the uncertain case, a motorcycle is used as the obstacle, and the LiDAR is mounted at a different height, creating out-of-distribution inputs. The EDL model estimates prediction uncertainty, allowing us to analyze how the distributionally robust safety constraint varies with different levels of uncertainty. For simplicity, $x$ and $y$ positions $(c_1^\xi, c_2^\xi)$ are considered with uncertainty distribution and $z$ coordinate is ignored because the height information of ground obstacles has minimal impact on collision avoidance.

Our proposed DRO-EDL-MPC algorithm generates trajectories of the ego vehicle to track waypoints. We use the same MPC parameter settings as described in \cite{hakobyan_distributionally_2023}. We use a kinematic bicycle model with control input $\mathbf{u}(t)$  consisting of acceleration and steering angle. The step-wise cost is $c(\mathbf{x}(t), \mathbf{u}(t))=\|\mathbf{x}(t)-p_t\|^2_Q+\|\Delta \mathbf{u}(t)\|^2_R$ to minimize the distance to nearby waypoints $p_t$ and reduce control input changes $\Delta \mathbf{u}(t)=\mathbf{u}(t)-\mathbf{u}(t-1)$. The terminal cost is $q(\mathbf{x}(T))=\|\mathbf{x}(T)-p_T\|^2_Q$, minimizing the distance between the ego state and the waypoint $p_T$ at the end of the horizon. We set $T=40, Q=diag(1,1,0,0.2),$ and $R=diag(1.5, 3)$.
The unknown dynamics model is learned via GPR using $M=50$ samples. This MPC problem is solved using ForcesPro \cite{zanelli_forces_2020}. 
To construct other baselines, we change the safety constraints within this common setting.
\begin{figure}[t]
\begin{center}
\includegraphics[width=1\linewidth]{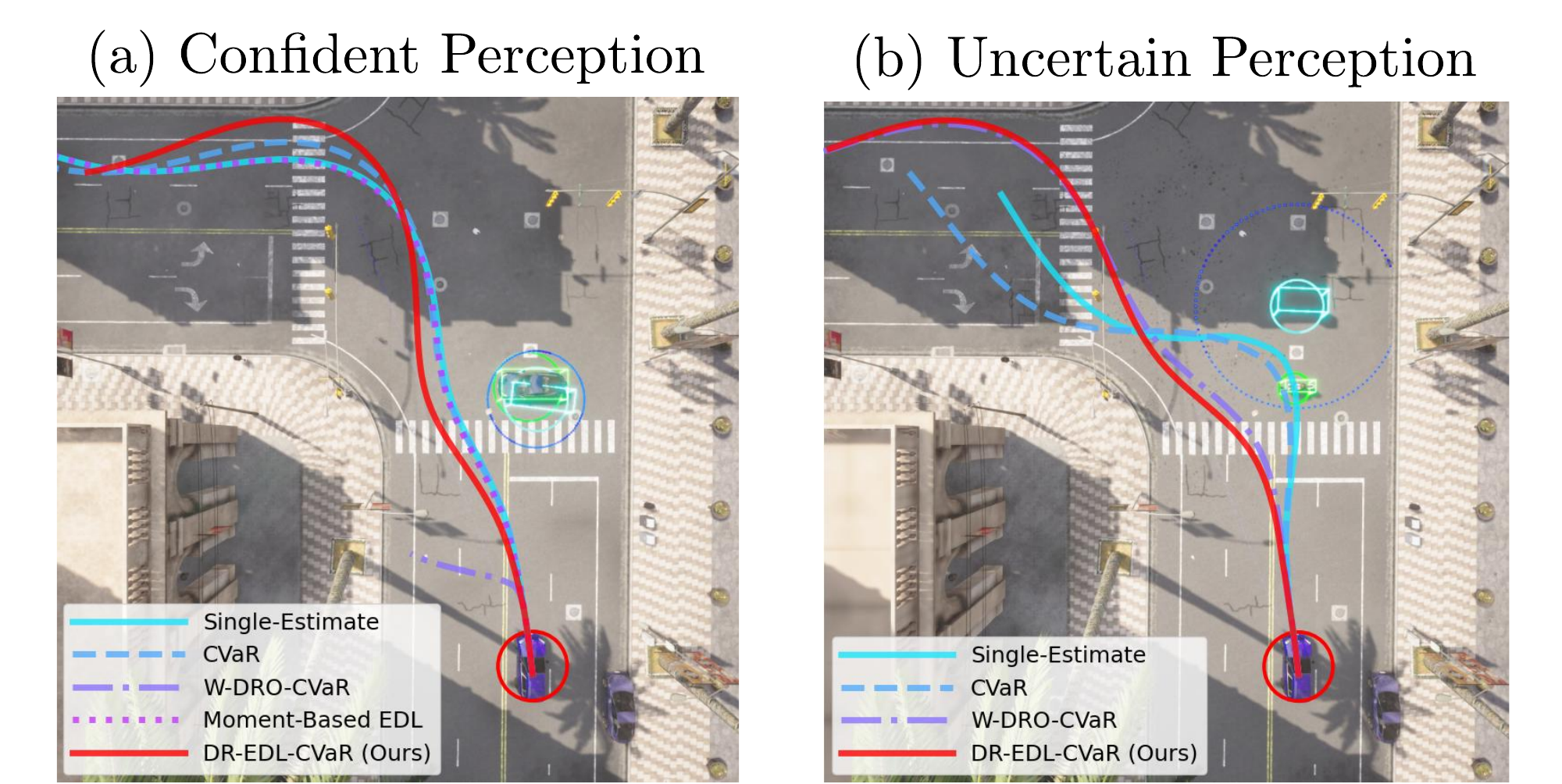}
\vspace*{-0.2in}
\end{center}
\caption{Comparison in (a) confident and (b) uncertain perception scenarios. 
Ego position (red), obstacle position (green), estimated obstacle position (cyan), and safety boundary (blue) are visualized.}
\vspace*{-0.25in}
\label{fig:simulation_result}
\end{figure}
We compare the following safety constraint baselines. 
\begin{itemize}
    \item \textbf{(Single-Estimate)} uses a single prediction of the obstacle state and thus a deterministic safety constraint. 
    \item \textbf{(CVaR)} models the data distribution as $\mathcal{N}(\mathbb{E}[\mu], \mathbb{E}[\sigma^2])$ based on the EDL estimation and formulates the corresponding CVaR risk metric.
    \item \textbf{(W-DRO-CVaR \cite{hakobyan_distributionally_2023})} enforces distributionally robust CVaR considering all distributions whose Wasserstein distance from the loss distribution $\mathbb{P}_{loss}$ is within 0.1.
    \item \textbf{(Moment-Based EDL \cite{wang_signal-devices_2025})} 
    enforces distributionally robust CVaR, utilizing the variance of the moments $Var[\mu], Var[\sigma^2]$ estimated from the EDL model as a confidence interval of the ambiguity set with a confidence coefficient of 0.5. 
    \item \textbf{(DR-EDL-CVaR)} is the proposed distributionally robust safety constraint with $\eta=0.9$. 
\end{itemize}
All stochastic methods use CVaR with $\epsilon=0.9$. 

We evaluate all methods over 100 runs. Specifically, we measure a collision rate with obstacles and a success rate of reaching the destination without collision. Among successful runs, we also report the average total cost and the average minimum distance to the estimated obstacle center $\|c^x-c^\xi\|_2$,
and the average optimization time.
\begin{figure}[t]
\begin{center}
\includegraphics[width=1\linewidth]{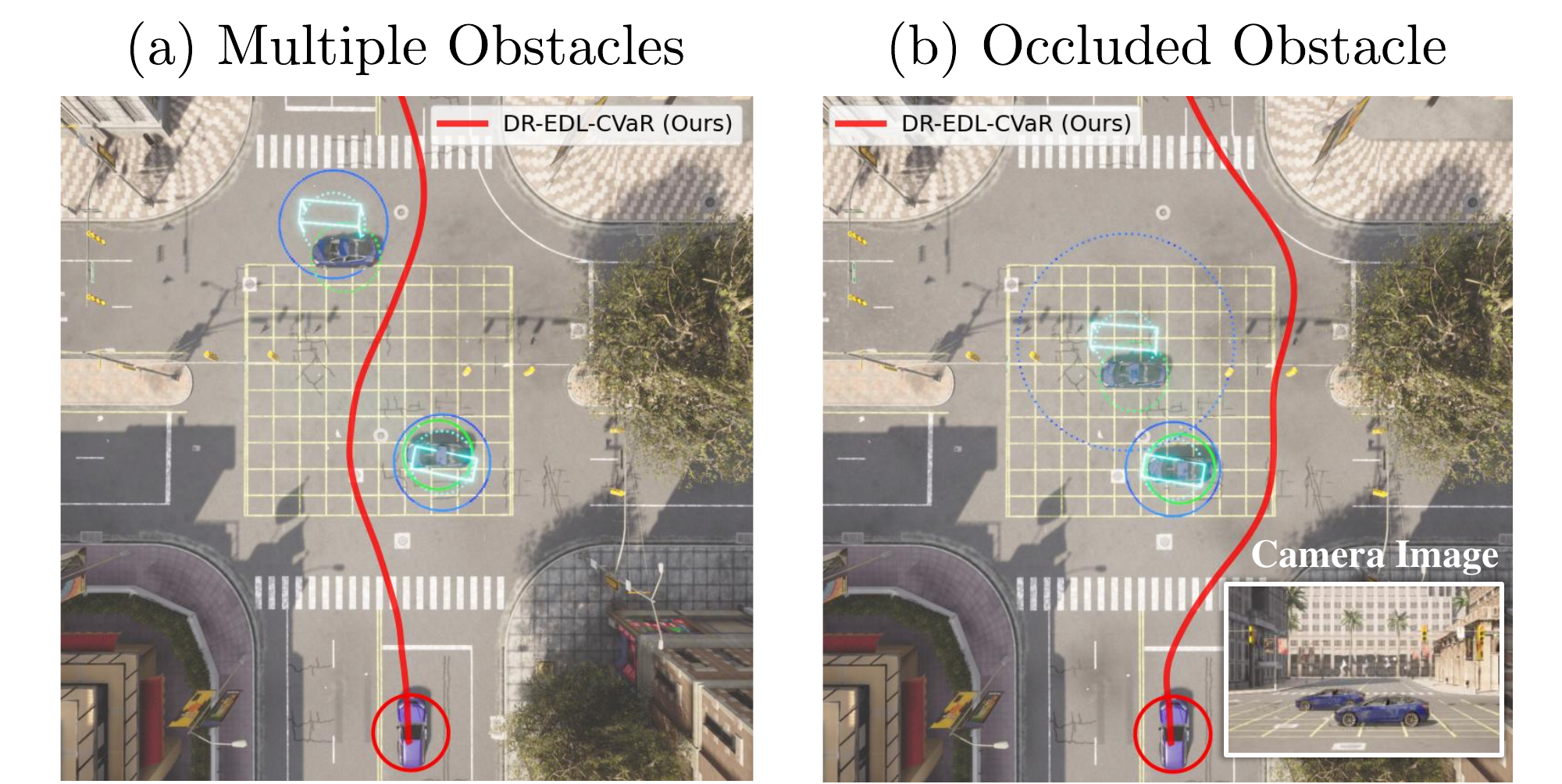}
\vspace*{-0.35in}
\end{center}
\caption{Illustration of our method in (a) multiple-obstacles and (b) occluded-obstacle scenarios.}
\label{fig:ablation}
\end{figure}

\begin{table}[h]
    \caption{Comparison in confident \& uncertain scenarios}
    \label{tab:result_1}
    \centering
    \resizebox{\linewidth}{!}{
    \begin{tabular}{cc c c c c c}
        \toprule
        Scenario & Method & {Succ.(\%) $\uparrow$} & {Coll. (\%) $\downarrow$} & {Cost $\downarrow$} & {Distance} & {Time (ms) $\downarrow$} \\
        \midrule
        \multirow{5}{*}{Confident} 
        & Single-Estimate   &   100 &   0   & $6.576\times 10^5$  & 2.625 & 220.5 \\
        & CVaR              &   100  &   0   & $1.585\times 10^{6}$ & 4.442 & 201.8 \\
        
        & W-DRO-CVaR            &   29  &   0   & $3.786\times 10^{11}$ & 7.672 & 1064  \\
        & Moment-Based EDL  &   78  &   0   & $2.281\times 10^7$  & 2.640 & 219.7 \\
        & DR-EDL-CVaR (Ours)    &   100 &   0   & $6.072\times 10^7$  & 5.032 & 203.5 \\
        \midrule
        \multirow{5}{*}{Uncertain} 
        & Single-Estimate   &   0     & 100 &  -                  &   -      & 306.4 \\
        & CVaR              &   53    &  44 &  $1.191\times 10^{5}$ &  6.848  & 199.5 \\
        & W-DRO-CVaR             &   65    &   2 &  $1.144\times 10^{11}$ &   11.436& 1168  \\
        & Moment-Based EDL  &     -   &   - &   -                &    -    &  -    \\
        & DR-EDL-CVaR (Ours)    &   95    &   2 &  $6.831\times 10^5$  &   8.014 &  223.1\\
        \bottomrule
    \end{tabular}
    }
\vspace*{-0.15in}

\end{table}

\subsection{Results}

The results of the confident perception experiments are illustrated in Fig. \ref{fig:simulation_result} (a) and Table \ref{tab:result_1}. All methods accurately infer obstacle positions, thereby causing no collisions. The Single-Estimate and CVaR approaches successfully reach destinations in all runs while maintaining low cost and minimal distance.
In contrast, the W-DRO-CVaR method maintains larger distances from the obstacle and shows a low success rate as it leads to overly conservative behavior and infeasibility due to the use of the unscented transform. The Moment-Based EDL method yields a similar conservative distance to the Single-Estimate approach, but suffers from a low success rate. This is because it is only applicable when perception is highly confident ($\alpha>2$). 
Our DR-EDL-CVaR approach constructs a small ambiguity set under confident perception, leading to a slight increase in the distance and cost compared to CVaR, thereby revealing a clear performance–conservativeness trade-off. At the same time, it remains less conservative and achieves a higher success rate than W-DRO-CVaR. 

The results of uncertain perception experiments are shown in Fig. \ref{fig:simulation_result} (b) and Table \ref{tab:result_1}. The uncertainty stems from model uncertainty because the perception model encounters a motorbike despite being trained only on the car class. The Single-Estimate and CVaR methods exhibit high collision rates due to their inability to handle model uncertainty.
In contrast, W-DRO-CVaR and our proposed method consider model uncertainty using the ambiguity set, resulting in significantly lower collision rates. 
The Moment-Based EDL method is inapplicable in all uncertain perception scenarios.
Our method attains the highest success rate, but still fails in 5\% of cases under uncertain perception. This is due to rare observations with extremely high uncertainty that make the controller overly conservative and cause it to get stuck.

Our DR-EDL-CVaR method is the only one that operates less conservatively under confident perception and more conservatively under uncertain perception, ensuring both efficiency and safety across varying perception confidence levels. These results confirm that our safety constraint formulation proves effective for safe autonomous driving.
Additionally, our method demonstrates a 5 times improvement in computational efficiency compared to the W-DRO-CVaR method while comparable to other simpler approaches.

\subsection{Multiple Obstacles}
We have also validated our method in scenarios with multiple obstacles and with occluded obstacles, as shown in Fig.~\ref{fig:ablation}. The results show that the method readily extends to multi-obstacle collision avoidance and provides a clear advantage in the occluded-obstacle case. When the rear obstacle is partially hidden by the front one, the resulting perception uncertainty enlarges the adaptive safety boundary, ensuring safe and robust operation.




\section{CONCLUSION}

We have proposed a safe motion planning approach that integrates EDL-based perception in uncertain environments. Our constraint formulation method, DR-EDL-CVaR, ensures safety through a distributionally robust safety constraint that accounts for the estimated uncertainty. Additionally, we have introduced DRO-EDL-MPC, a conservative yet tractable motion planning algorithm based on the upper bound of the distributionally robust safety loss and standardization of the NIG distribution. Experiments demonstrate that our approach is the only method that can produce safe motion in out-of-distribution scenarios without additional training, while exhibiting low conservativeness in in-distribution scenarios and achieving great efficiency compared to other methods. Our future work will address multiple dynamic obstacles in real-world environment settings.








\printbibliography

\ifarxiv
\section*{APPENDIX}
\subsection{Properties of the Value-at-Risk}
\begin{lemma} \label{lemma:VaR}
\textnormal{(Properties of the Value-at-Risk).}
Given a random variable $X\sim\mathbb{P}=\mathcal{N}(\mu, \sigma^2)$ and a threshold $\epsilon\in [0,1)$, the following holds. 
\begin{enumerate}[label=(\alph*)]
    \item $\text{VaR}_{\epsilon}^\mathbb{P}[-|X|^2]=-\text{VaR}_{\epsilon}^\mathbb{P}[-|X|]^2$
    \item For $Y=-|X|$ with a fixed $\sigma^2$, 
    \begin{equation*}
        \begin{aligned}
            \frac{d \text{VaR}_{\epsilon}^\mathbb{P}[Y]}{d\mu}=\frac{\phi(\frac{\text{VaR}_{\epsilon}^\mathbb{P}[Y]-\mu}{\sigma})-\phi(\frac{-\text{VaR}_{\epsilon}^\mathbb{P}[Y]-\mu}{\sigma})}{\phi(\frac{\text{VaR}_{\epsilon}^\mathbb{P}[Y]-\mu}{\sigma})+\phi(\frac{-\text{VaR}_{\epsilon}^\mathbb{P}[Y]-\mu}{\sigma})},
        \end{aligned}
    \end{equation*}
    where $\phi$ is the probability density function of the standard normal distribution.
    \item If $\mu=0$, then $\text{VaR}_\epsilon^\mathbb{P} [-|X|]=\sqrt{2}\sigma\cdot erf^{-1}(\epsilon-1)$
    \item $\text{VaR}_{\epsilon}^\mathbb{P}[X]> \text{VaR}_{\epsilon}^\mathbb{P}[-|X|]$
\end{enumerate}
\end{lemma}
\begin{proof}
(a)
\begin{equation*} 
    \begin{aligned}
        &P[-|X|\leq\text{VaR}^\mathbb{P}_{\epsilon}[-|X|]]=\epsilon, \\
        &P[-|X|^2\leq\text{VaR}^\mathbb{P}_{\epsilon}[-|X|^2]]=\epsilon \\
        \Leftrightarrow &P[|X|\geq\sqrt{-\text{VaR}^\mathbb{P}_{\epsilon}[-|X|^2]}]=\epsilon\\
         \Leftrightarrow &P[-|X|\leq-\sqrt{-\text{VaR}^\mathbb{P}_{\epsilon}[-|X|^2]}]=\epsilon.
    \end{aligned}
\end{equation*}
\begin{equation*}
    \therefore \text{VaR}_{\epsilon}^\mathbb{P}[-|X|^2]=-\text{VaR}_{\epsilon}^\mathbb{P}[-|X|]^2
\end{equation*}

(b) 
Let $F_Y(y)$ and $f_Y(y)$ be the cumulative probability distribution and the probability density function of a random variable $Y$, respectively. Let us similarly define $F_X(x)$ and $f_X(x)$ for a random variable $X$. Also, let $\Phi$ be the cumulative distribution function of the standard normal distribution.
For simplicity, we denote $\text{Var}_\epsilon^\mathbb{P}[Y]$ by $k(\mu),$ which is also a function of $\sigma$ but we consider it as a constant.

By definition of VaR, we have 
\begin{equation*}
    \begin{aligned}
    &{\epsilon}=F_Y(k(\mu))=\int_{-\infty}^{k(\mu)}f_Y(y)dy \\
     &=\int_{-\infty}^{k(\mu)}(f_X(x)+f_X(-x))dx \\
     &=\int_{-\infty}^{k(\mu)}f_X(x)dx+\int_{-\infty}^{k(\mu)}f_X(-x)dx. \\
    \end{aligned}
\end{equation*}
Let $t=-x$. Then,
\begin{equation*}
    \begin{aligned}
    &F_Y(k(\mu))=\int_{-\infty}^{k(\mu)}f_X(x)dx+\int_{\infty}^{-k(\mu)}f_X(t)(-dt) \\
    &=\int_{-\infty}^{k(\mu)}f_X(x)dx-\int_{\infty}^{-k(\mu)}f_X(x)dx \\
    &=\Phi\left(\frac{k(\mu)-\mu}{\sigma}\right)-\Phi\left(\frac{-k(\mu)-\mu}{\sigma}\right)+1. \\
    \end{aligned}
\end{equation*}
Because ${\epsilon}=F_Y(k(\mu))$, we have $\frac{d F_Y(k(\mu))}{d\mu}=0$. This leads to
\begin{multline*}
    \phi\left(\frac{k(\mu)-\mu}{\sigma}\right)\cdot\left(\frac{d k(\mu)}{d\mu}-1\right)\frac{1}{\sigma} \\
    -\phi\left(\frac{-k(\mu)-\mu}{\sigma}\right)\cdot\left(-\frac{d k(\mu)}{d\mu}-1\right)\frac{1}{\sigma}=0.
\end{multline*}
Then
\begin{equation*}
    \begin{aligned}
    &\frac{dk(\mu)}{d\mu}=\frac{\phi\left(\frac{k(\mu)-\mu}{\sigma}\right)-\phi\left(\frac{-k(\mu)-\mu}{\sigma}\right)}{\phi\left(\frac{k(\mu)-\mu}{\sigma}\right)+\phi\left(\frac{-k(\mu)-\mu}{\sigma}\right)}. \\
    \end{aligned}
\end{equation*}
This concludes the proof.

(c) Let $-|X|=Y$. Then the probability density function of $Y$ is
\begin{equation*}
    \begin{aligned}
        f_Y(y)=f_X(x)+f_X(-x)=\frac{\sqrt{2}}{\sigma\sqrt{\pi}}exp(-\frac{y^2}{2\sigma^2}), y\leq 0.
    \end{aligned}
\end{equation*}
And the cumulative distribution function of $Y$ is 
\begin{equation*}
    \begin{aligned}
        F_Y(y)=\int_{-\infty}^y f_Y(t)dt=\int_{-\infty}^y\frac{\sqrt{2}}{\sigma\sqrt{\pi}}exp(-\frac{t^2}{2\sigma^2})dt.
    \end{aligned}
\end{equation*}
Meanwhile,
\begin{equation*}
    \begin{aligned}
        F_Y(0)=\int_{-\infty}^0 f_Y(t)dt=1
    \end{aligned}
\end{equation*}
because $Y$ is defined in $(-\infty,0]$. Then
\begin{equation*}
    \begin{aligned}
        &F_Y(0)=\int_{-\infty}^y f_Y(t)dt + \int_y^0 f_Y(t)dt \\
        &=\int_{-\infty}^y f_Y(t)dt - \int_0^y f_Y(t)dt.
    \end{aligned}
\end{equation*}
Therefore, 
\begin{equation*}
    \begin{aligned}
        \int_{-\infty}^y f_Y(t)dt &= 1+\int_0^y f_Y(t)dt\\
        &=1+\int_{0}^y\frac{\sqrt{2}}{\sigma\sqrt{\pi}}exp(-\frac{t^2}{2\sigma^2})dt.
    \end{aligned}
\end{equation*}
Let ${t}/{\sqrt{2}\sigma}=k.$ Then
\begin{equation*}
    \begin{aligned}
        & F_Y(y)=1+\frac{\sqrt{2}}{\sigma\sqrt{\pi}}\int_{0}^{{y}/{\sqrt{2}\sigma}}exp(-k^2)\cdot \sqrt{2}\sigma dk \\
        & =1+\frac{2}{\sqrt{\pi}}\int_{0}^{{y}/{\sqrt{2}\sigma}}exp(-k^2)dk \\
        & =1+erf({y}/{\sqrt{2}{\sigma}}).
    \end{aligned}
\end{equation*}
$\text{VaR}_{\epsilon}^\mathbb{P}[Y]$ is the value of $y$ when $F_Y(y)=\epsilon$. Therefore,
\begin{equation*}
    \begin{aligned}
        \text{VaR}_\epsilon^\mathbb{P}[Y]=\sqrt{2}\sigma\cdot erf^{-1}(\epsilon-1).
    \end{aligned}
\end{equation*}

(d) Let $\text{VaR}_{\epsilon}^\mathbb{P}[X]=k, \text{VaR}_{\epsilon}^\mathbb{P}[-|X|]=k'$.
By definition, 
\begin{equation*}
    \int^{\infty}_k f_X(x) dx={1-\epsilon}=\int^0_{k'} f_{-|X|}(x)dx.
\end{equation*}
Also,
\begin{equation*}
    \begin{aligned}
        &\int^{\infty}_k f_X(x)dx \\
        &=\int_k^0f_X(x)dx + \int_0^{-k} f_X(x)dx + \int^{\infty}_{-k}f_X(x)dx \\
        &=\int_k^0f_X(x)dx + \int_k^{0} f_X(-x)dx + \int^{\infty}_{-k}f_X(x)dx \\
        &=\int_k^0f_{-|X|}(x)dx + \int^{\infty}_{-k}f_X(x)dx.
    \end{aligned}
\end{equation*}
Therefore,
\begin{equation*}
    \begin{aligned}
        &\int^\infty_{-k} f_X(x)dx >0\\
        \implies & \int_{k'}^0f_{-|X|}(x)dx - \int_{k}^0f_{-|X|}(x)dx \\
        &=\int_{k'}^kf_{-|X|}(x)dx>0.       
    \end{aligned}
\end{equation*}
Therefore, $\text{VaR}_{\epsilon}^\mathbb{P}[X]>\text{VaR}_{\epsilon}^\mathbb{P}[-|X|].$
\end{proof}
\subsection{Properties of the Conditional Value-at-Risk}

\setcounter{lemma}{0}
\begin{lemma} \label{lemma}
\textnormal{(Properties of the Conditional Value-at-Risk).}
Given a random variable $X\sim\mathbb{P}=\mathcal{N}(\mu, \sigma^2)$ and a threshold $\epsilon\in [0,1)$, the following holds. 
\begin{enumerate}[label=(\alph*)]
    \item $\text{CVaR}_{\epsilon}^\mathbb{P}[-|X|^2]\leq -\text{CVaR}_{\epsilon}^\mathbb{P}[-|X|]^2$
    \item For a fixed $\sigma^2$, $\text{CVaR}_{\epsilon}^\mathbb{P}[-|X|]$ is monotonically decreasing with respect to $|\mu|$.
    \item If $\mu=0$, then $\text{CVaR}_\epsilon^\mathbb{P}[-|X|]=\kappa\cdot\sigma$, \\
    where $\kappa=\frac{1}{1-\epsilon}\sqrt{{2}/{\pi}}[exp(-[erf^{-1}(\epsilon-1)]^2)-1]$.
    \item If $\text{VaR}_\epsilon^\mathbb{P}[X]> 0 \land \mu> 0$, or $\text{VaR}_\epsilon^\mathbb{P}[X]< 0 \land \mu< 0$, then $-|\mu|+\delta\cdot\sigma> \text{CVaR}_\epsilon^\mathbb{P}[-|X|]$.
\end{enumerate}
\end{lemma}
\begin{proof}
    
(a)
\begin{align}
    \text{CVaR}_{\epsilon}^\mathbb{P}[-|X|^2]&=\mathbb{E}[-|X|^2:-|X|^2\geq\text{VaR}_{\epsilon}^\mathbb{P}[-|X|^2]] \\ 
    &=\mathbb{E}[-|X|^2:-|X|^2\geq -\text{VaR}_{\epsilon}^\mathbb{P}[-|X|]^2]  \label{eq:18} \\
    &=\mathbb{E}[-|X|^2:|X|\leq-\text{VaR}_{\epsilon}^\mathbb{P}[-|X|]] \\
    &=-\mathbb{E}[|X|^2:|X|\leq-\text{VaR}_{\epsilon}^\mathbb{P}[-|X|]] \\
    &\leq -\mathbb{E}[-|X|:|X|\leq-\text{VaR}_{\epsilon}^\mathbb{P}[-|X|]]^2  \label{eq:20} \\
    &= -\mathbb{E}[-|X|:-|X|\geq\text{VaR}_{\epsilon}^\mathbb{P}[-|X|]]^2 \\
    &=-\text{CVaR}_{\epsilon}^\mathbb{P}[-|X|]^2.
\end{align}
Equation \eqref{eq:18} holds by Lemma \ref{lemma:VaR} (a), and inequality \eqref{eq:20} holds by Jensen's inequality.

(b) With the same notations used in the proof of Lemma~\ref{lemma:VaR} (b), we have
\begin{equation*}
    \begin{aligned}
        \text{CVaR}_{\epsilon}^\mathbb{P}[Y]&=\mathbb{E}[Y: Y\geq k(\mu)] =\frac{1}{{1-\epsilon}}\int_{k(\mu)}^0yf_Y(y)dy \\
        &=\frac{1}{{1-\epsilon}}\int_{k(\mu)}^0x(f_X(x)+f_X(-x))dx \\
        &=\frac{1}{{1-\epsilon}}\left[\int_{k(\mu)}^0 xf_X(x)dx+\int_{-k(\mu)}^0tf_X(t)dt\right] \\
        &=\frac{1}{{1-\epsilon}}\left[\int_{k(\mu)}^0xf_X(x)dx-\int^{-k(\mu)}_0xf_X(x)dx\right]. 
    \end{aligned}
\end{equation*}
Let $a(\mu), b(\mu)$ be functions with respect to $\mu$, and standardized random variable $Z=(X-\mu)/\sigma$. Then,
\begin{equation*}
    \begin{aligned}
    \int_{a(\mu)}^{b(\mu)}xf_X(x)dx&=\int_{\frac{a(\mu)-\mu}{\sigma}}^{\frac{b(\mu)-\mu}{\sigma}}(\mu+\sigma\cdot z) \frac{1}{\sigma}\phi(z) \sigma dz \\
        &=\mu\int_{\frac{a(\mu)-\mu}{\sigma}}^{\frac{b(\mu)-\mu}{\sigma}}\phi(z)dz + \sigma\int_{\frac{a(\mu)-\mu}{\sigma}}^{\frac{b(\mu)-\mu}{\sigma}}z\phi(z)dz. \\
    \end{aligned}
\end{equation*}
With $z_a(\mu)=\frac{a(\mu)-\mu}{\sigma}$ and $z_b(\mu)=\frac{b(\mu)-\mu}{\sigma}$, the derivative with respect to $\mu$ is
\begin{multline*}
        \frac{d}{d\mu}\int_{a(\mu)}^{b(\mu)}xf_X(x)dx= \\
        \left[\int_{z_a(\mu)}^{z_b(\mu)}\phi(z)dz+\mu\cdot \{\phi(z_b(\mu))\cdot {z'_b(\mu)} -\phi(z_a(\mu))\cdot z'_a(\mu)\} \right] \\
        +\sigma\left[ z_b(\mu) \phi(z_b(\mu))z'_b(\mu)-z_a(\mu)\phi(z_a(\mu))z'_a(\mu)\right] \\
        =\int_{z_a(\mu)}^{z_b(\mu)}\phi(z)dz+ b(\mu)\phi(z_b(\mu)) z'_b(\mu)-        a(\mu)\phi(z_a(\mu)) z'_a(\mu).
\end{multline*}
Using this result and letting $z_k(\mu)=\frac{k(\mu)-\mu}{\sigma}, z_{-k}(\mu)=\frac{-k(\mu)-\mu}{\sigma}$ and $z_0(\mu)=\frac{-\mu}{\sigma}$, we can express the derivative of  $\text{CVaR}_\epsilon^\mathbb{P}[Y]$ with respect to $\mu$ as follows.
\begin{align*}
        (1-&\epsilon)\frac{d}{d\mu}\text{CVaR}_\epsilon^\mathbb{P}[Y]\\
        &=\frac{d}{d\mu}\left[\int_{k(\mu)}^0xf_X(x)dx-\int_0^{-k(\mu)}xf_X(x)dx\right] \\
        &=\int^{z_0(\mu)}_{z_k(\mu)}\phi(z)dz - k(\mu)\cdot\phi(z_k(\mu))\cdot z'_k(\mu) \\
        &~~-\int^{z_{-k}(\mu)}_{z_0(\mu)}\phi(z)dz -\left[ (-k(\mu))\cdot\phi(z_{-k}(\mu))\cdot z'_{-k}(\mu)\right] \\
        &=\int^{z_0(\mu)}_{z_k(\mu)}\phi(z)dz-\int^{z_{-k}(\mu)}_{z_0(\mu)}\phi(z)dz \\
        &~~-k(\mu)\left[\phi(z_k(\mu))\cdot z'_k(\mu)-\phi(z_{-k}(\mu))\cdot z'_{-k}(\mu)\right]. 
    \end{align*}
Because 
\begin{align*}
    z'_k(\mu)=\frac{k'(\mu)-1}{\sigma}, && z'_{-k}(\mu)=\frac{-k'(\mu)-1}{\sigma},
\end{align*}
and by Lemma \ref{lemma:VaR} (b),
\begin{align*}
k'(\mu)=\frac{d k(\mu)}{d\mu}=\frac{\phi(z_k(\mu))-\phi(z_{-k}(\mu))}{\phi(z_k(\mu))+\phi(z_{-k}(\mu))}, 
\end{align*} we have $\phi(z_k(\mu))\cdot z'_k(\mu)-\phi(z_{-k}(\mu))\cdot z'_{-k}(\mu)=0$. Therefore,
\begin{equation*}
    \begin{aligned}
    &\frac{d}{d\mu}\text{CVaR}_\epsilon^\mathbb{P}[Y]=\frac{1}{1-\epsilon}\left[\int^{\frac{-\mu}{\sigma}}_{\frac{k(\mu)-\mu}{\sigma}}\phi(z)dz-\int^{\frac{-k(\mu)-\mu}{\sigma}}_{\frac{-\mu}{\sigma}}\phi(z)dz\right].
    \end{aligned}
\end{equation*}
Note that for any $b\leq0$, $\int_{a+b}^a \phi(z)dz> \int^{a-b}_a \phi(z)dz$ if $a> 0$ and $\int^a_{a+b} \phi(z)dz< \int^{a-b}_a \phi(z)dz$ if $a<0$.
Therefore,
\begin{equation*}
    \begin{aligned}
    &\frac{d}{d\mu}\text{CVaR}_{\epsilon}^\mathbb{P}[Y]<0~\text{for }\mu>0.\\
    \end{aligned}
\end{equation*}
and
\begin{equation*}
    \begin{aligned}
    &\frac{d}{d\mu}\text{CVaR}_{\epsilon}^\mathbb{P}[Y]>0~\text{for }\mu<0.\\
    \end{aligned}
\end{equation*}
That is, $\text{CVaR}_{\epsilon}^\mathbb{P}[Y]$ is monotonically decreasing with respect to $|\mu|$.

(c) Let $-|X|=Y$ and $\text{VaR}_\epsilon[Y]=k$. Then
\begin{equation*}
    \begin{aligned}
        & \text{CVaR}_\epsilon^\mathbb{P}[Y]=\mathbb{E}[Y:Y\geq k] \\
        & =\frac{1}{1-\epsilon}\int_{k}^0 yf_Y(y)dy \\
        & =\frac{1}{1-\epsilon}\int^0_{k} y\cdot\frac{\sqrt{2}}{\sqrt{\pi}\sigma}exp(-\frac{y^2}{2\sigma^2})dy.
    \end{aligned}
\end{equation*}
Let ${y}/{\sqrt{2}\sigma}=u, y=\sqrt{2}\sigma u, dy=\sqrt{2}\sigma du$. Then
\begin{equation*}
    \begin{aligned}
        & =\frac{1}{1-\epsilon}\int^0_{{k}/{\sqrt{2}\sigma}} \sqrt{2}\sigma u\cdot\frac{\sqrt{2}}{\sqrt{\pi}\sigma}exp(-u^2)\sqrt{2}\sigma du \\
        & =\frac{1}{1-\epsilon}\cdot \frac{2\sqrt{2}}{\sqrt{\pi}}\sigma\int^0_{{k}/{\sqrt{2}\sigma}}u\cdot exp(-u^2)du \\
        & =\frac{1}{1-\epsilon}\cdot \frac{2\sqrt{2}}{\sqrt{\pi}}\sigma\left[-\frac{1}{2}exp(-u^2)\right]^0_{{k}/{\sqrt{2}\sigma}} \\
        & =\frac{1}{1-\epsilon}\cdot \frac{\sqrt{2}}{\sqrt{\pi}}\sigma\left[exp(-\left[{k}/{\sqrt{2}\sigma}\right]^2)-1\right] \\
        & =\frac{1}{1-\epsilon}\cdot \frac{\sqrt{2}}{\sqrt{\pi}}\sigma\left[exp(-\left[\frac{\sqrt{2}\sigma\cdot erf^{-1}(\epsilon-1)}{\sqrt{2}\sigma}\right]^2)-1\right] \\
        & =\frac{1}{1-\epsilon}\cdot \frac{\sqrt{2}}{\sqrt{\pi}}\sigma\left[exp(-[erf^{-1}(\epsilon-1)]^2)-1\right] \\
        & =\kappa\cdot\sigma,
    \end{aligned}
\end{equation*}
where $k=\sqrt{2}\sigma\cdot erf^{-1}(\epsilon-1)$ by Lemma~\ref{lemma:VaR} (c) and $\kappa=\frac{1}{1-\epsilon}\sqrt{{2}/{\pi}}\left[exp(-[erf^{-1}(\epsilon-1)]^2)-1\right]$.

(d) If $\mu< 0$ and $\text{VaR}_{\epsilon}^\mathbb{P}[X]=k< 0$, then 
\begin{equation*}
    \begin{aligned}
        & \text{CVaR}_{\epsilon}^\mathbb{P}[X]=\mathbb{E}[X|X\geq k] \\
        &=\frac{1}{{1-\epsilon}}\int^{\infty}_{k}xf_X(x)dx \\
        &> \frac{1}{{1-\epsilon}}\int^{-k}_k xf_X(x)dx \\
        &=\frac{1}{{1-\epsilon}}[\int_{k}^0 xf_X(x)dx + \int_0^{-k} xf_X(x)dx] \\
        &=\frac{1}{{1-\epsilon}}[\int_{k}^0 xf_X(x)dx + \int_0^{k} xf_X(-x)dx] \\
        &=\frac{1}{{1-\epsilon}}[\int_{k}^0 xf_X(x)dt - \int_k^0 xf_X(-x)dx] \\
        &\geq\frac{1}{{1-\epsilon}}[\int_{k}^0 x f_X(x)dt + \int_k^0 xf_X(-x)dx] \\
        &=\frac{1}{{1-\epsilon}}\int_{k}^0 x(f_X(-x)+f_X(x))dx \\
        &=\frac{1}{{1-\epsilon}}\int_{k}^0 xf_{-|X|}(x)dx \\
        &\geq\frac{1}{1-\epsilon}\int^{0}_{\text{VaR}_{\epsilon}^\mathbb{P}[-|X|]} xf_{-|X|}(x)dx \\
        &=\mathbb{E}[-|X|:-|X|\geq \text{VaR}_{\epsilon}^\mathbb{P}[-|X|]] \\
        &=\text{CVaR}_{\epsilon}^\mathbb{P}[-|X|]. \\
    \end{aligned}
\end{equation*}
The last inequality is because $k > \text{VaR}_\epsilon^{\mathbb{P}}[-|X|]$ by Lemma~\ref{lemma:VaR} (d).
Therefore, $\text{CVaR}_\epsilon^\mathbb{P}[X]=\mu+\delta\cdot\sigma>\text{CVaR}_\epsilon^\mathbb{P}[-|X|].$

Similarly, if $\mu> 0$ and $\text{VaR}_\epsilon^\mathbb{P}[X]> 0$, the mean of $-X$ is smaller than zero and $\text{VaR}_\epsilon^\mathbb{P}[-X]< 0$. Thus, we can use the inequality that we derived in the preceding paragraph, that is,
\begin{equation*}
    \begin{aligned}
        &\text{CVaR}_\epsilon^\mathbb{P}[-X] >\text{CVaR}_\epsilon^\mathbb{P}[-|X|].
        \end{aligned}
        \end{equation*}
Because $X$ has a symmetric distribution about the mean $\mu$, we have $-\text{VaR}_\epsilon^\mathbb{P}[-X]=\text{VaR}_{1-\epsilon}^\mathbb{P}[X]$. Also, $\text{CVaR}_{1-\epsilon}^\mathbb{P}[X]=\mathbb{E}[X:X\geq\text{VaR}_\epsilon^\mathbb{P}[X]]=\mu-\delta\cdot\sigma$ by \cite{norton_calculating_2021}. Thus, 
  \begin{equation*}
    \begin{aligned}      
        \text{CVaR}_\epsilon^\mathbb{P}[-X]&=\mathbb{E}[-X:-X\geq\text{VaR}_\epsilon^\mathbb{P}[-X]] \\
        &=-\mathbb{E}[X:X\leq-\text{VaR}_\epsilon^\mathbb{P}[-X]]\\
        &=-\mathbb{E}[X:X\leq\text{VaR}_{1-\epsilon}^\mathbb{P}[X]]\\
        &=-(\mu-\delta\cdot\sigma).
    \end{aligned}
\end{equation*}Therefore, if $\text{VaR}_\epsilon^\mathbb{P}[X]>0\land\mu>0$ or $\text{VaR}_\epsilon^\mathbb{P}[X]<0\land\mu< 0$, then $-|\mu|+\delta\cdot\sigma> \text{CVaR}_\epsilon^\mathbb{P}[-|X|]$.
\end{proof}

\fi

\end{document}